\documentclass[11pt,draftcls,onecolumn,journal]{IEEEtran}
\usepackage[utf8]{inputenc} 
\usepackage[T1]{fontenc}    
\usepackage{url}            
\usepackage{booktabs}       
\usepackage{amsfonts}       
\usepackage{nicefrac}       
\usepackage{microtype}      
\usepackage{xcolor,colortbl}         

\definecolor{mygray}{RGB}{211,211,211}
\definecolor{cvprblue}{RGB}{0.21,0.49,0.74}
\definecolor{mypurple}{RGB}{100, 50, 168}
\usepackage[pagebackref,breaklinks,colorlinks,citecolor=cvprblue,linkcolor=cvprblue]{hyperref}

\usepackage{times}
\usepackage{epsfig}
\usepackage{graphicx}
\usepackage{tikz}
\usepackage{tikz-3dplot}
\usetikzlibrary {arrows.meta}
\usetikzlibrary{matrix,fit}

\usepackage{ifthen}

\usetikzlibrary{decorations.markings, decorations.pathmorphing, patterns, shapes.geometric}

\usetikzlibrary{positioning}

\usepackage{subfig}

\usepackage[linesnumbered,lined,boxed,commentsnumbered,ruled,vlined]{algorithm2e}

\SetCommentSty{mycommfont}

\usepackage{comment}

\usepackage{booktabs}
\usepackage{multirow}

\usepackage{amsfonts}
\usepackage{bm}
\usepackage{amsmath,amsthm,amssymb,rotating, mathrsfs}

\usepackage[capitalise]{cleveref}
\usepackage{crossreftools}
\crefformat{equation}{(#2#1#3)}

\usepackage{footmisc}

\usepackage{mathtools}
\usepackage{cancel} 

\usepackage{xfrac} 
\usepackage{xparse} 

\DeclareMathAlphabet{\mathpzc}{OT1}{pzc}{m}{it}
\usepackage{enumitem}

\newcommand{\appallingunderline}[1]{%
	\underline{\smash{#1}\vphantom{T}}\vphantom{#1}%
}

\newcommand{\cE}{\mathcal{E}}
\newcommand{\cF}{\mathcal{F}}

\newcommand{\cN}{\mathcal{N}}

\newcommand{\bbE}{\mathbb{E}}

\newcommand{\bbR}{\mathbb{R}}

\NewDocumentCommand{\norm}{mG{2}}{\big\|#1\big\|_{#2}}

\DeclareMathOperator{\diag}{diag}
\DeclareMathOperator{\rank}{rank}

\newcommand{\argmin}{\mathop{\rm argmin}}
\newcommand{\argmax}{\mathop{\rm argmax}}

\NewDocumentCommand{\seqp}{mG{n}}{{#1}_1-\cdots+ {#1}_{#2}}
\NewDocumentCommand{\seqm}{mG{n}}{{#1}_1-\cdots- {#1}_{#2}}

\newcommand{\myparagraph}[1]{\textbf{#1.}}

\DeclareMathOperator{\trace}{tr}

\newtheorem{theorem}{Theorem}

\newtheorem{prop}{Proposition}

\newtheorem{lemma}{Lemma}

\theoremstyle{definition}
\newtheorem{definition}{Definition}

\newtheorem{assumption}{Assumption}

\crefname{assumption}{Assumption}{Assumptions} 
\Crefname{assumption}{Assumption}{Assumptions} 

\newtheorem{example}{Example}

\theoremstyle{remark}
\newtheorem{remark}{Remark}

\definecolor{cvprblue}{rgb}{0.21,0.49,0.74}
\usepackage[pagebackref,breaklinks,colorlinks,citecolor=cvprblue,linkcolor=cvprblue]{hyperref}

\definecolor{myred}{rgb}{0.82, 0.1, 0.26}
\begin{document}

\title{Mathematics of Continual Learning}

\author{Liangzu Peng \quad \quad Ren\'e Vidal,~\IEEEmembership{Fellow,~IEEE}
\thanks{This document is prepared by the Special Issue Area Editor, Selin Aviyente  and by the Editor-in-Chief, T\"{u}lay~Adal{\i}.}
\thanks{Selin Aviyente is with the Department of Electrical and Computer Engineering, Michigan State University.}
\thanks{Special Issue Area Editor email: SPM-SI-AREA@LISTSERV.IEEE.ORG.}
\thanks{Manuscript received July 1, 2024.}}

\markboth{IEEE Signal Processing Magazine,~Vol.~XX, No.~XX, June~2024}%
{Aviyente \MakeLowercase{\textit{et al.}}: Author Guidelines for Special Issue Articles of IEEE SPM}

\maketitle

Continual learning is an emerging subject in machine learning that aims to solve multiple tasks presented sequentially to the learner without forgetting previously learned tasks. Recently, many deep learning based approaches have been proposed for continual learning, however the mathematical foundations behind existing continual learning methods remain underdeveloped. On the other hand, adaptive filtering is a classic subject in signal processing with a rich history of mathematically principled methods. However, its role in understanding the foundations of continual learning has been underappreciated. In this tutorial, we review the basic principles behind both continual learning and adaptive filtering, and present a comparative analysis that highlights multiple connections between them. These connections allow us to enhance the mathematical foundations of continual learning based on existing results for adaptive filtering, extend adaptive filtering insights using existing continual learning methods, and discuss a few research directions for continual learning suggested by the historical developments in adaptive filtering.

\section{Introduction}

\subsection{Continual Learning and Deep Continual Learning}
Throughout their lives, humans need to solve multiple tasks---addressing them one at a time or, at times, simultaneously. The tasks could be learning to sit, crawl, stand, walk,  run, and bike. During the process, humans acquire knowledge and learn continually. They might forget what they previously learned, but if needed, they could regain it much more quickly than when they learned it the first time.

Machines, particularly machine learning models, need to solve multiple tasks, too. \textit{Multitask learning} aims to achieve so by joint training on all data of all tasks. However, in many cases we do not have access to all data or all tasks at once; rather, data often come in a streaming fashion, and a new task might arise only when the current one is sufficiently learned (e.g., humans learn to walk before running). Learning multiple tasks that are presented sequentially is the main goal of \textit{continual learning} (also known as \textit{lifelong learning}, or \textit{incremental learning}). A typical continual learning pipeline has two stages. In the first stage, a model is learned for task 1, which is typically achieved via standard machine learning approaches. In the second stage, the model is updated as new tasks arrive sequentially, ideally without forgetting previous tasks. This is where a continual learning method is often designed and tailored. 

\begin{example}[\textit{Continual Learning of Large Language Models}]\label{example:ChatGPT}
    ChatGPT was released on November 30, 2022. With a large amount of new data available after its initial release, how can we update ChatGPT from its previous version, as efficiently and economically as possible, without affecting its original performance? 
\end{example}

Given a new task, a continual learning method is usually designed with two goals in mind: (1) improve its performance on the new task; and (2) maintain its performance on previous tasks. This often requires striking the right balance, as illustrated by the following two extreme cases. \textit{Case 1}: If we keep the model unchanged, we attain goal (2), but we may fail at goal (1), e.g., if the new data has a different distribution from that of the previously seen data. \textit{Case 2}: If we train on the new data until convergence using standard algorithms such as \textit{stochastic gradient descent} (SGD), we could attain goal (1), but we may fail at goal (2), e.g., if the distribution of the data is different, even if we initialize at the previously learned model. In the latter case, the phenomenon called \textit{catastrophic forgetting} \cite{Nccloskey-1989} could occur: The new task is learned at the cost of significant performance degradation on the previous tasks. 

A direct method to overcome catastrophic forgetting is to perform multitask learning as a new task arrives, but it assumes full access to past data and can be both memory and computationally inefficient. \textbf{Replay-based methods} address these issues by storing only part of past data, used together with new data for training. Methods of this type are also called \textit{rehearsal}. In \textit{Psychology}, rehearsal is considered a \textit{control process} that maintains \textit{short-term memory} and is key for humans to grow \textit{long-term memory} \cite{Atkinson-1968}. But in our machine learning context, replaying past data might still forget when a new task begins to be trained \cite{Nccloskey-1989}. This phenomenon has recently been rediscovered and called the \textit{stability gap} \cite{Lange-ICLR2023}.

While rehearsal often puts equal weights on the losses of different tasks, \textbf{constrained optimization methods} emphasize the past by solving the current task under the constraints that past tasks are solved reasonably well  \cite{Lopez-NeurIPS2017} or to global optimality \cite{Peng-ICML2023}. For example, the \textit{Ideal Continual Learner} (ICL) of \cite{Peng-ICML2023} prevents forgetting by design, and can be implemented using a primal-dual constrained learning framework \cite{Chamon-TIT2022}. ICL can also be implemented by \textbf{gradient projection methods}, which modify the standard training of the current task by projecting the gradient onto a carefully chosen subspace \cite{Farajtabar-AISTATS2020,Saha-ICLR2021,Wang-CVPR2021,Min-CDC2022}. 

While constrained optimization and gradient projection use extra memory to work with the constraints and projection matrices, respectively, regularization- and expansion-based methods entail almost no extra storage (except the models themselves). \textbf{Regularization-based methods} optimize the current task with some regularization term that promotes proximity to a previously learned model; this philosophy recently finds its signature use in \textit{reinforcement learning with human feedback} \cite{Ziegler-arXiv2019}, where the current task is to align the model behavior with human values. \textbf{Expansion-based methods} add new trainable parameters to learn new tasks; one of the earliest such methods is \cite{Ash-CS1989}, where new neurons are continually added to the deep network as the training proceeds---until the given task is solved to a desired accuracy.

The recent rise of \textit{foundational models for vision and language} 
has steered research on deep continual learning towards at least two directions. One is where, in addition to overcoming catastrophic forgetting, the emphasis is on continual learning of larger and larger models \textit{efficiently} (\cref{example:ChatGPT}). 
The other direction involves attaching trainable parameters to a large, pre-trained model, enabling continual learning of downstream tasks; this approach, by definition, belongs to the category of expansion-based methods. While the performance continues to improve along either of the directions, the mathematical reasons behind the improvements are only becoming more elusive as models get larger and more complicated.


\subsection{History of Adaptive Filtering and Our Conjectured Connection with Continual Learning }
To the best of our knowledge, the 1989 paper \cite{Nccloskey-1989} is commonly cited as one of the earliest works in continual learning. This work shows that (shallow) \textit{connectionist networks}, trained by backpropagation, exhibit catastrophic forgetting in a sequential learning setting. However, the rudimentary idea of continual learning dates back to at least 1960 \cite{LMS-1960}, when Widrow \& Hoff proposed a single-layer neural network called \textit{Adaline} (\textit{Adaptive Linear Element}), trained via \textit{Least Mean Squares} (LMS). The design of LMS was guided by the \textit{minimal disturbance principle}, that is (quoting  \cite{Widrow-IEEE1990}), ``\textit{adapt to reduce the output error for the current training pattern, with minimal disturbance to responses already learned}'' or ``\textit{... inject new information into a network in a manner that disturbs stored information to the smallest extent possible}''. How is this not an early description of \textit{continual learning without forgetting}?

In \cite{Widrow-IEEE1990} it was also recollected that, after some unsuccessful attempts to devise learning algorithms for multi-layer neural networks, Widrow shifted his research focus to \textit{adaptive filtering}, a subject that has since become a classic in signal processing. LMS has found so many applications there that it is often regarded as \textit{the birthmark of modern adaptive filtering theory} (quoting \cite{Sayed-2008book}). These adaptive filtering applications are typically concerned with an online linear regression setting, where consecutive regressor vectors are highly correlated (i.e., the current regressor is a \textit{shifted} version of the previous regressor with an entry removed and a new entry added). This differs from continual learning, as the latter is typically concerned with nonlinear models (in particular, deep networks) and does not exhibit such data correlation.

Despite these differences, we argue that 
adaptive filtering has deep yet often overlooked connections to continual learning---connections that extend far beyond their shared origin (arguably, LMS). The key observations that allow us to connect them are three-fold. First, there has been a growing body of work exploring theoretical aspects of continual learning under linear models; we find some of them extend the classic LMS insights. Second, representative adaptive filtering methods, such as LMS, \textit{Affine Projection Algorithm} (APA), \textit{Recursive Least-Squares} (RLS), and  \textit{Kalman Filter} (KF), are trivially extended for the case where the aforementioned data correlation does not exist. Third, these methods can furthermore be extended for nonlinear models or deep networks in three practical ways: (1) linearize the nonlinear models, (2) apply these methods in a layer-wise fashion, or (3) train a linear classifier continually (using output features of some pre-trained model); we find several such extensions in the continual learning literature. Next, we describe the forenamed adaptive filtering methods (LMS, APA, RLS, KF) and the corresponding continual learning methods that share the same or similar philosophy (see \cref{table:algorithm-connection}).


\textbf{LMS} was originally proposed for training Adaline \cite{LMS-1960}. It processes one sample at a time and is often interpreted as an online SGD-type method for linear regression. A basic result is this \cite[Section 5.5.1]{Theodoridis-2020}: If there is a \textit{true linear model} that perfectly explains all samples, and if all samples are \textit{statistically the same}, then LMS converges to the true model with a proper stepsize. Alternatively, we equate \textit{time} to \textit{task}, and interpret LMS as a continual learning method for a sequence of linear regression tasks; namely, LMS processes tasks sequentially, with one sample per task. This interpretation brings the basic LMS guarantee into contact with the recent work in continual learning and optimization \cite{Evron-COLT2022,Peng-ICML2024}, whose results imply that LMS variants converge to the true model if all linear regression tasks are \textit{cyclically repeated}. These findings are not at odds with the empirical observation that SGD forgets catastrophically in deep continual learning, as in practice the tasks are neither statistically the same nor cyclically repeated. The requirement of revisiting (statistically) the same tasks to find the true model is a natural consequence of being \textit{memoryless}, i.e., LMS stores no extra information except the previous model (and stepsize).


\textbf{APA} was originally proposed for \textit{learning identification} \cite{Hinamoto-1975} and adaptive filtering \cite{Ozeki-1984} and was intended to attain faster convergence (e.g., than LMS). For the current task (or at the current time), APA receives a new sample; it projects the previous model onto the constraint set defined by this sample and several past samples. Since APA and its variants reuse past data, they are referred to as the \textit{data-reusing family} in adaptive filtering, which corresponds to replay-based methods in continual learning. We will show an APA variant is connected to ICL \cite{Peng-ICML2023} and gradient projection methods \cite{Saha-ICLR2021,Wang-CVPR2021,Min-CDC2022}, as they all find the same optimal (minimum-norm) solution to linear regression tasks (a phenomenon that might be called \textit{implicit bias}). Their optimality illustrates the suitable use of extra memory that prevents forgetting. 


\begin{table}[]
    \centering
    \caption{Adaptive filtering methods and their counterparts that we find in continual learning. }
    \label{table:algorithm-connection}
    \begin{tabular}{ll}
        \toprule
          Adaptive Filtering & Continual Learning  \\
         \midrule
         \multirow{2}*{\quad \textit{Least Mean Squares} (LMS) \cite{LMS-1960}}
           & \quad Connection to SGD for continual linear regression \cite{Evron-COLT2022,Peng-ICML2024} \\ 
           & \quad \textit{Perspective}: LMS is a constrained optimization method \\
          \midrule
          \multirow{4}*{\quad \textit{Affine Projection Algorithm} (APA) \cite{Hinamoto-1975,Ozeki-1984}}
            & \quad Connection to Ideal Continual Learner (ICL) \cite{Peng-ICML2023} \\ 
            & \quad Connection to gradient projection methods  \cite{Farajtabar-AISTATS2020,Saha-ICLR2021,Wang-CVPR2021,Min-CDC2022} \\
           & \quad \textit{Perspective}: APA is a constrained optimization method \\ 
           & \quad \textit{Perspective}: APA is a replay-based method  \\
          \midrule
          \multirow{4}*{\quad \textit{Recursive Least-Squares} (RLS)} & \quad Layer-wise RLS for deep continual learning \cite{Shah-1992,Zeng-NMI2019} \\ 
          & \quad RLS for dynamically expanding networks \cite{Azimi-TNN1993,Zhuang-NeurIPS2022}  \\ 
          & \quad \textit{Perspective}: RLS is a regularization-based method \\ 
          & \quad \textit{Perspective}: RLS forgets the past via the \textit{forgetting factor} \\
          \midrule
          \quad  \multirow{3}*{\textit{Kalman Filter} (KF) \cite{Kalman-1960}} & \quad KF for continual learning of linear classifiers \cite{Titsias-ICLR2024} \\ 
          & \quad \textit{Perspective}: Linear Gaussian models introduce a task relationship \\ 
          & \quad \textit{Perspective}: KF facilitates \textit{positive backward transfer} \\ 
         \bottomrule
    \end{tabular}
\end{table}

\textbf{RLS} dates back to the work of Gauss (see, e.g., the historical accounts in \cite{Sayed-2008book}). It continually minimizes a weighted least-squares loss (with regularization), where the weights for previous samples decay exponentially via a \textit{forgetting factor}. In other words, RLS minimizes a weighted average of the losses of all seen tasks (in the continual linear regression setting with one sample per task); in comparison, LMS only considers the current loss, and APA puts past losses in the constraints (corresponding to infinitely large weights). Decades ago, RLS was applied to train neural networks in a layer-wise fashion;  for a survey, see, e.g., \cite{Shah-1992}. This was recently repurposed for deep continual learning to reduce forgetting \cite{Zeng-NMI2019}. The RLS idea was also leveraged to train dynamically expanding networks in 1993 \cite{Azimi-TNN1993} and 2022 \cite{Zhuang-NeurIPS2022}. The former does so for every layer, in a layer-wise manner, while the latter does so for the last-layer classifier and goes by the name \textit{analytic continual learning}. However, by its least-squares nature, RLS might give a poor estimate for every task, if the solutions to different tasks are far from each other.


\textbf{KF} comes with a \textit{linear Gaussian model}, which posits that the current \textit{state} is a linear function of the previous \textit{state} (\textit{state transition}) and each measurement is a linear function of the corresponding state (up to noise). In short, KF is a method that tracks the \textit{most recent} state based on \textit{all} measurements. We note that the measurement-state relation defines a (noisy) linear regression task where the state is the unknowns, and we interpret state transition as \textit{task transition}, which defines a \textit{linear task relationship}; namely, the solution to the current task is a linear function of the solution to the previous task (up to noise). Different from the setting of RLS, the linear Gaussian model allows multiple samples per task, and the linear task relationship allows the solutions of different tasks to be arbitrarily far (as determined by the linear map that defines this relationship); we will see KF generalizes RLS. Furthermore, we will show combining KF with the so-called \textit{Rauch-Tung-Striebel smoother} (RTS) provably \textit{improves the performance on previous tasks when learning the current one}, a situation called \textit{positive backward transfer} in continual learning. The KF idea was used to train neural networks in 1988 (see survey \cite{Shah-1992}), and to train linear classifiers with pre-trained models in 2024 \cite{Titsias-ICLR2024}. While KF requires the task relationship to be known, it is possible to learn this relationship in conjunction with model parameters (e.g., \cite{Titsias-ICLR2024}).

Building upon the basics of linear algebra, probability, optimization, machine learning, we deliver in what follows a mathematical review and comparative analysis of the methods in \cref{table:algorithm-connection}. Since the tutorial uniquely illuminates the connections of adaptive filtering and continual learning, we refer the reader to books \cite{Sayed-2008book,Theodoridis-2020} and surveys \cite{Van-NMI2022,Zhou-TPAMI2024,Wang-TPAMI2024} for a wider coverage of the respective subjects.

\section{Least Mean Squares (LMS)}
\myparagraph{Notations and Setup}  Suppose we have an initial model $\theta^0\in \bbR^d$, be it pre-trained, randomly initialized, or equal to $0$. At task $t\geq 1$, we are given a single input-output data pair $(x_t,y_t)\in \bbR^{d}\times \bbR$, and our goal is  to update $\theta^{t-1}$ into a new model $\theta^t$. In this case, the total number of tasks $T$ is equal to the total number of data pairs. For conciseness, we stack the data together and write $X_{:t}:=[x_1,\dots,x_t]\in \bbR^{d\times t}, y_{:t}:=[y_1,\dots,y_t]^\top$. For simplicity, we make the following assumption ($\| \cdot \|_2$ denotes the $\ell_2$ norm of a vector):
\begin{assumption}\label{assumption:linear-unit-norm}
    For all $t=1,\dots,T$, we have $\| x_t\|_2 = 1$ and $\rank(X_{:t})=t\leq d$. Denote by $\theta^*$ the \textit{minimum-norm solution} to the equations $y_{:T}= X_{:T}^\top \theta$ in variable $\theta$, that is $\theta^{*}=X_{:T} (X_{:T}^\top X_{:T})^{-1} y_{:T}$. 
\end{assumption} 
\cref{assumption:linear-unit-norm} considers the overparameterized case $T \leq d$, where the equations $y_{:T}= X_{:T}^\top \theta$ always admit a solution. These equations define a \textit{linear model}, as  $y_{:T}$ is a \textit{linear} function of the \textit{model parameters} $\theta$. The (minimum-norm) solution to these equations, $\theta^*$, is referred to as the \textit{ground-truth} or \textit{true model} (\textit{parameters}). We will use the \textit{solution of all tasks} interchangeably with the \textit{true model}. 

\myparagraph{Metrics} Under the linear model assumption, it makes sense to define the error $\epsilon_{ij}$ as follows: 
\begin{align}\label{eq:err-mat}
	\epsilon_{ij}:=y_i - x_i^\top \theta^j. \quad \forall i,j=1,\dots,T.
\end{align}
This gives  a $T\times T$ matrix of error terms, and each $\epsilon_{ij}$ measures the error of the model learned at task $j$ on the data of task $i$. The matrix is visualized in \cref{fig:regret-finetune-cl}, where we further discuss two other paradigms:
\begin{itemize}
	\item In \textit{online optimization} (or \textit{online learning}) \cite{Orabona-arXiv2019}, the metric of interest is the \textit{prediction error} or \textit{a priori error} $\epsilon_t:=y_t - x_t^\top  \theta^{t-1}$, where model $\theta^{t-1}$ was obtained \textit{without} knowing $(x_t,y_t)$. In this line of research, one often aims to bound the \textit{regret} $\sum_{i=1}^{t} \epsilon_t^2$ as $t$ grows.
	\item In \textit{finetuning} (or \textit{transfer learning}), one updates a pre-trained model for a new task. The new model is expected to have low errors on the new task, but not necessarily on the past pre-training data.	
\end{itemize}


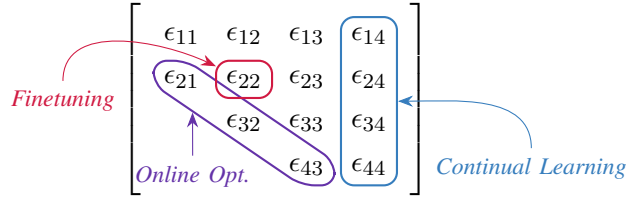
\begin{figure}
	\centering
	\begin{tikzpicture}[
		every matrix/.style={
			matrix of math nodes,
			column sep =1mm,
			row sep =1mm,
			left delimiter={[},
			right delimiter={]},
		}
		]
		\matrix (m) {
			\epsilon_{11} & \epsilon_{12} & \epsilon_{13} & \epsilon_{14} \\
			\epsilon_{21} & \epsilon_{22} & \epsilon_{23} & \epsilon_{24} \\
			& \epsilon_{32} & \epsilon_{33} & \epsilon_{34} \\
			&  & \epsilon_{43} & \epsilon_{44} \\
		};
		
		\draw[thick,rounded corners=5pt, mypurple]
		([yshift=-2pt]m-2-1.west) |- 
		([xshift=2pt]m-2-1.north) -- 
		([yshift=2pt]m-4-3.east) |- 
		([xshift=-2pt]m-4-3.south) -- 
		cycle;
		\node[fit=(m-1-4)(m-4-4), draw, inner sep=0pt,thick, cvprblue,rounded corners=5pt] (moo) {};
		\node[fit=(m-2-2)(m-2-2), draw, inner sep=0pt,thick, myred, rounded corners=5pt] {};
		
		\node[anchor=east] at (-2, 0) (ft) {\footnotesize{\textit{\textcolor{myred}{Finetuning}}}};
		\draw[-latex, myred, -{Stealth[length=2mm]}] (ft) to[out=90,in=155] (m-2-2);
		
		\node[anchor=east] at (-0.2, -1) (oo) {\footnotesize{\textit{\textcolor{mypurple}{Online Opt.}}}};
		\draw[-latex, mypurple, -{Stealth[length=2mm]}] (oo) to[out=90,in=270] (-1.1, -0.15);

		\node[anchor=west] at (2, -0.9) (cl) {\footnotesize{\textit{\textcolor{cvprblue}{Continual Learning}}}};
		\draw[-latex, cvprblue, -{Stealth[length=2mm]}] (cl) to[out=90,in=0] (moo);
	\end{tikzpicture}
	\caption{Visualization of a $4\times 4$ error matrix where $\epsilon_{ij}$ is the error of model $\theta^j$ on task $i$, as defined in \cref{eq:err-mat}. \textit{Online opitmization} considers errors $\epsilon_{ij}$ with $j=i-1$. \textit{Finetuning} considers errors $\epsilon_{ii}$ on the current task. Continual learning considers errors on both the current and previous tasks. \label{fig:regret-finetune-cl} }
\end{figure}

In continual learning, we want to measure the error of the current model $\theta^t$ on both the current and past data. A commonly used metric, known as the \textit{forgetting}, is $\cF_t:= \frac{1}{t-1} \sum_{i=1}^{t-1} \left(\epsilon_{it}^2 - \epsilon_{ii}^2 \right)$. Each summand $\epsilon_{it}^2 - \epsilon_{ii}^2$ measures the error of model $\theta^t$ on a previous task $i$, in comparison to model $\theta^i$. Since model $\theta^i$ is learned immediately after task $i$ is given, we expect $\theta^i$ to perform reasonably well (if not perfectly) on task $i$. Thus, we expect $\epsilon_{ii}^2$ to be close to $0$ and $\cF_t$ close to the \textit{mean squared error} $\cE_t$ (MSE), defined as
\begin{align*}
	\cE_t:=\frac{1}{t} \sum_{i=1}^{t}   \epsilon_{it}^2 =\frac{1}{t} \sum_{i=1}^{t}  \left(y_i - x_i^\top  \theta^t \right)^2.
\end{align*}
In fact, since $\cF_t\leq \frac{t}{t-1}\cdot \cE_t$, we will derive upper bounds for the (expected) MSE $\cE_t$ in the sequel and understand that these bounds (multiplied by $\frac{t}{t-1}$) will hold for the forgetting $\cF_t$ as well. 

\myparagraph{LMS} Recall $\epsilon_t:=y_t - x_t^\top  \theta^{t-1}$. The LMS method \cite{LMS-1960} amounts to solving (cf. \cite[Section 10.4]{Sayed-2008book})
\begin{align}\label{eq:LMS}
	\theta^t \in \argmin_{\theta \in \bbR^d} \| \theta - \theta^{t-1} \|_2^2 \quad \quad \textnormal{s.t.} \quad \quad y_t - x_t^\top  \theta = \left(1-\gamma \right) \epsilon_t, \tag{LMS}
\end{align}
where $\gamma$ is a hyper-parameter that satisfies $1-\gamma \in (-1,1)$. With this $\gamma$, the \textit{a posterior error} $y_t - x_t^\top  \theta^t $ is always smaller than the a priori error $\epsilon_t$ in magnitude, enabling learning from the present task $t$. The objective  $\| \theta - \theta^{t-1} \|_2^2$ is minimized in an effort to preserve the past knowledge (implicitly stored in $\theta^{t-1}$). Balancing the past and present in this way concretizes the aforementioned minimal disturbance principle. 

Geometrically, \ref{eq:LMS} projects $\theta^{t-1}$ onto the affine hyperplane defined by its linear constraint. \ref{eq:LMS} is always feasible and it is an exercise to show \ref{eq:LMS} admits the following closed-form solution:
\begin{align}\label{eq:LMS-SGD}
	\theta^t =\theta^{t-1} - \gamma \left(x_t^\top  \theta^{t-1} - y_t \right) x_t.
\end{align}
This might be seen as an online SGD update with constant stepsize $\gamma$ for the objective $(x_t^\top \theta- y_t)^2/2$. We next derive upper bounds on the MSE $\cE_t$ of the \ref{eq:LMS} iterates $\{\theta^t\}_{t= 0}^T$. 




\myparagraph{Expected MSE} Since $y_i=x_i^\top \theta^*$ and $\| x_i  \|_2=1$, using the \textit{Cauchy–Schwarz} inequality we have
\begin{equation}
		\cE_t = \frac{1}{t} \sum_{i=1}^{t}  \left(x_i^\top  \theta^* - x_i^\top  \theta^t \right)^2  
		\leq  \frac{1}{t}  \sum_{i=1}^{t} \| x_i  \|_2^2 \cdot  \| \theta^t -  \theta^*  \|_2^2 
		=  \| \theta^t -  \theta^*  \|_2^2.
\end{equation}
We then need to bound $\| \theta^t -  \theta^*  \|_2^2$. To do so, let us derive a recurrence relation between $\theta^t -  \theta^*$ and $\theta^{t-1} -  \theta^*$. Since $y_t=x_t^\top \theta^*$, rewrite \cref{eq:LMS-SGD} as
\begin{align*}
	\theta^t = \theta^{t-1} - \gamma \left(x_t^\top  \theta^{t-1} - x_t^\top  \theta^*  \right) x_t.
\end{align*}
Subtract $\theta^*$ from both sides of the above equality, and we obtain ($I_d$ denotes the $d\times d$ identity matrix)
\begin{align}\label{eq:kaczmarz-update2}
	\theta^t - \theta^* = \left(I_d- \gamma \cdot x_t x_t^\top\right) \left( \theta^{t-1} -  \theta^*  \right).
\end{align}
We are ready to bound the expected MSE loss $\bbE[\cE_t]$:
\begin{theorem}[Section 5.5.1 \cite{Theodoridis-2020}]\label{theorem:kaczmarz}
    Let $\{\theta^t\}_{t=0}^T$ be the iterates of \ref{eq:LMS} with $\theta^0=0$, $\gamma\in(0,2)$. Assume $x_t$'s are independent and identically distributed, drawn according to some distribution on the sphere $\{x\in \bbR^d: \| x\|_2 = 1\}$. Write $\Sigma_x:= \bbE[x_tx_t^\top]$, and assume $\Sigma_x$ is positive definite with minimum eigenvalue $\lambda_{\textnormal{min}}(\Sigma_x)$. Under 
    \cref{assumption:linear-unit-norm}, we have $\gamma(2-\gamma)\lambda_{\textnormal{min}}(\Sigma_x)\in(0,1]$ and
	\begin{align*}
		\bbE[\cE_t] \leq \bbE\| \theta^t -  \theta^*  \|_2^2 \leq \Big(1 - \gamma(2-\gamma)\lambda_{\textnormal{min}}(\Sigma_x) \Big)^t \cdot \| \theta^*  \|_2^2.
	\end{align*}
\end{theorem}
\begin{proof}
	We have shown the first inequality. Note that $\gamma(2-\gamma)$ and all eigenvalues of $\Sigma_x$  lie in $(0,1]$, and so do all eigenvalues of $\gamma(2-\gamma)\Sigma_x$. Thus $I_d- \gamma(2-\gamma) \cdot \Sigma_x$ is positive semi-definite. 
	From \cref{eq:kaczmarz-update2} we get
	\begin{align*}
		\bbE\| \theta^t -  \theta^*  \|_2^2 
		&=\bbE \left[ (\theta^{t-1} - \theta^*)^\top \big(I_d- \gamma(2-\gamma) \cdot \Sigma_x\big) (\theta^{t-1} -  \theta^*)  \right] \\ 
		&\leq \Big(1 - \gamma(2-\gamma)\lambda_{\textnormal{min}}(\Sigma_x) \Big) \cdot  \bbE\| \theta^{t-1} -  \theta^*  \|_2^2,  
	\end{align*}
        where the first equality holds as $x_t$ is independent of $\theta_{t-1},\theta^*$. The proof is finished by unrolling.
\end{proof}
\cref{theorem:kaczmarz} shows that $\theta^t$ converges to the true model $\theta^*$ in expectation, exponentially, by the \textit{contraction factor} $1 - \gamma(2-\gamma)\lambda_{\textnormal{min}}(\Sigma_x)$. The expected MSE loss $\bbE[\cE_t]$ decays at a similar rate. 

\myparagraph{Deterministic MSE} 
We now consider \textit{recurring} or \textit{periodic} tasks as motivated in \cite{Evron-COLT2022}:
\begin{definition}
	We say the $T$ tasks are $p$-recurring if $x_t= x_{t+p}$ for every $t=1,\dots,T-p$.
\end{definition}
In the $p$-recurring setting, there are at most $p$ distinct tasks, presented in a cyclic fashion, and iteratively applying \cref{eq:kaczmarz-update2} becomes an instance of the (generalized) alternating projection algorithm (and also related to the classic \textit{Kaczmarz method}); see \cite{Evron-COLT2022,Peng-ICML2024} and the references therein. We have: 
\begin{theorem}[Theorems 8, 10 \& Lemma 9 \cite{Evron-COLT2022}]\label{theorem:ap}
	Assume the $T$ tasks are $2$-recurring with $T$ even. Let $\{\theta^t\}_{t=0}^T$ be the iterates of \ref{eq:LMS} with $\theta^0=0$, $\gamma=1$. Assume $c:=(x_1^\top x_2)^2\neq 1$.  \cref{assumption:linear-unit-norm} implies
	\begin{align*}
		\|\theta^T - \theta^*\|_2^2 &\leq c^{T-1} \cdot \| \theta^* \|_2^2, \\ 
		\cE_T &\leq c^{T-1} \left( 1 - c\right) \cdot \| \theta^* \|_2^2 \leq \frac{\| \theta^* \|_2^2}{e(T-1)}.
	\end{align*}
\end{theorem} 
\begin{proof}
    Here we give a proof different from \cite{Evron-COLT2022} to the last inequality $ c^{T-1} \left( 1 - c\right)\leq \frac{1}{e(T-1)}$. 
    Applying the \textit{AM-GM inequality} to the $T$ numbers, $c/(T-1),\dots, c/(T-1), 1-c$, gives
	\begin{align*}
		\left( \frac{c}{T-1} \right)^{T-1} \left( 1 - c \right) \leq \left( \frac{1}{T} \right)^T.
	\end{align*}
	Hence we have $ c^{T-1} \left( 1 - c\right)\leq \left( 1 - \frac{1}{T}  \right)^T \cdot \frac{1}{T-1}$. To finish, note that $\left( 1 - \frac{1}{T}  \right)^T\leq e^{-1}$.
\end{proof}
\cref{theorem:ap} states that $\theta^T$ converges to $\theta^*$ exponentially with the contraction factor $c=(x_1^\top x_2)^2$. If $c$ is arbitrarily close to $1$, then the convergence is arbitrarily slow. On the other hand, the upper bound $\frac{\| \theta^* \|_2^2}{e(T-1)}$ on the MSE loss $\cE_T$ decreases with $T$ and is independent of $c$ (that is, independent of data).


\myparagraph{On Optimal Stepsizes}  The stepsize $\gamma=1$ in \cref{theorem:ap} is a crucial choice. With $\gamma=1$, the matrix $I_d- \gamma\cdot x_t x_t^\top$ in \cref{eq:kaczmarz-update2} becomes an orthogonal projection, whose properties could be leveraged for analysis. This stepsize choice might even seem to be optimal in light of \cref{theorem:kaczmarz}, as its upper bound is precisely minimized at $\gamma=1$. However, note that $\gamma$ is constant; what if we use different stepsizes $\gamma_t$ for different tasks? If we do so, \cref{eq:kaczmarz-update2} becomes $\theta^t - \theta^* = \left(I_d- \gamma_t \cdot x_t x_t^\top\right) \left( \theta^{t-1} -  \theta^*  \right)$. Setting  $\gamma_t=1$ yields the same bound as in \cref{theorem:ap}, but a better choice of $\gamma_t$ does exist (Corollary 1 \& Appendix D of \cite{Peng-ICML2024}):
\begin{theorem}\label{theorem:opt}
	Consider the same setting of \cref{theorem:ap} but use the following alternating stepsizes:
	\begin{align*}
		\gamma_t = \begin{cases}
			1 & \quad \textnormal{if $t$ is odd,} \\ 
			\frac{1}{1-c} & \quad \textnormal{if $t$ is even,}
		\end{cases}
	\end{align*}
	where we recall $c= (x_1^\top x_2)^2\neq 1$. Then \ref{eq:LMS} converges at task 3 with $\theta^3 = \theta^*$, $\cE_3=0$ (\cref{fig:ap-opt}).
\end{theorem}
\begin{proof}
	Recall $X_{:2}=[x_1, x_2]$, $y_{:2}=[y_1, y_2]^\top$. It suffices to show $\theta^3=\theta^*$, as this implies $\cE_3=0$. Since $\gamma_1=1$ and $\gamma_2 = \frac{1}{1-c}$, \cref{eq:kaczmarz-update2} indicates $\theta^3 - \theta^* =G(\theta^0-\theta^*) =-G \theta^*$, where $G$ is given as
	\begin{align*}
		G:=\left(I_d- x_1 x_1^\top\right) \left(I_d- \gamma_2 \cdot x_2 x_2^\top\right) \left(I_d- x_1 x_1^\top\right).
	\end{align*}
	By algebra one verifies $GX_{:2}=0$, which implies $\theta^3 - \theta^* = -G \theta^* = -G X_{:2}(X_{:2}^{\top} X_{:2} )^{-1}y_{:2}=0$.
\end{proof}

\begin{figure}
	\centering
		\begin{tikzpicture}[scale=1.2]
		\draw[solid, line width=1.5pt] (0, 0) -- (2,0) node[right] {\footnotesize{$\textnormal{Span}(x_1)^\perp$}};
		
		\draw[solid, line width=1.5pt] (0,0) -- (1,1.7321) node[right] {\footnotesize{$\textnormal{Span}(x_2)^\perp$}};
		
		\fill[cvprblue] (1.7321, 0.7) circle (1.5pt) node[above] {$z^0$};
		
		\draw[dashed, cvprblue, -{Stealth[length=2mm]}] (1.7321, 0.7)  -- (1.7321,0) node[below] {$z^1$};
		\draw[dashed, cvprblue, -{Stealth[length=2mm]}] (1.7321,0)  -- (0,1) node[above] {$z^2$};
		\draw[dashed, cvprblue, -{Stealth[length=2mm]}] (0,1) -- (0,0) node[below] {$z^3$};
		
		\draw[dotted, myred, -{Stealth[length=2mm]}] (1.7321, 0.7)  -- (1.7321,0);
		\draw[dotted, myred, -{Stealth[length=2mm]}] (1.7321,0)  -- (0.433,0.75);
		\node[myred] at (0.7,0.85) {$z^2$};
		
		\draw[dotted, myred, -{Stealth[length=2mm]}] (0.433,0.75) -- (0.433,0) node[below] {$z^3$};
	\end{tikzpicture}
	\caption{Convergence of $z^t:=\theta^t - \theta^*$ to the intersection $0$ with the constant stepsize $\gamma_t=1$ of \cref{theorem:ap} (dotted red arrows) or with the alternating stepsize of \cref{theorem:opt} (dashed blue arrows). \label{fig:ap-opt} }
\end{figure}
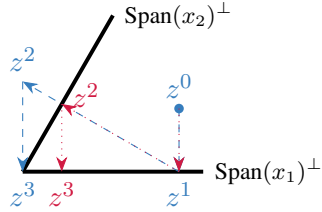
If $c\to1$ then $\gamma_2\to \infty$, so stepsize $\gamma_2$ is far larger than what \ref{eq:LMS} would choose to decrease the error. That said, finding optimal stepsizes for more than two recurring tasks remains an open problem.

\myparagraph{Summary} 
While we refer the readers to \cite{Evron-COLT2022,Peng-ICML2024} for results that generalize \cref{theorem:ap,theorem:opt}, our simplified exposition conveys that \ref{eq:LMS} has small MSE losses \textit{if it frequently encounters tasks that are previously seen}. Indeed, \cref{theorem:kaczmarz,theorem:ap,theorem:opt} require the tasks to be statistically the same or cyclically repeated. In the subsequent sections we will handle previously \textit{unseen} tasks by using extra memory.

\section{Affine Projection Algorithm (APA)}
\myparagraph{Notations} Denote by $\textnormal{Span}(\cdot)$ the subspace spanned by the columns of the input matrix, and by $\textnormal{Span}(\cdot)^\perp$ the orthogonal complement of this subspace. For example, $\textnormal{Span}(X_{:t})$ is spanned by $x_1,\dots,x_t$. 
Denote by $P_t$ the orthogonal projection onto $\textnormal{Span}(X_{:t})^\perp$; $P_t$ is a $d\times d$ symmetric  matrix, expressed as
\begin{align}\label{eq:proj-P}
	P_t:=I_d - X_{:t} (X_{:t}^\top X_{:t})^{-1} X_{:t}^\top.
\end{align}
This projection $P_t$ will play a fundamental role and offer geometric intuition in the section.




\myparagraph{APA} In the continual learning language, APA \cite{Hinamoto-1975,Ozeki-1984} is a \textit{replay-based method} that maintains a memory buffer of size $b$ to store $b$ most recent samples. Given a new task $t$ with data $(x_t,y_t)$ and the previous model $\theta^{t-1}$, and given the $b$ samples $\{ (x_{t-i},y_{t-i}) \}_{i=1}^{b}$ in the memory buffer, APA amounts to
\begin{align}\label{eq:APA}
	\theta^t \in \argmin_{\theta \in \bbR^d} \| \theta - \theta^{t-1}\|_2^2 \quad \textnormal{s.t.} \quad y_{t-i}=x_{t-i}^\top \theta,\ i=0,\dots,b. \tag{APA}
\end{align}
As shown, \ref{eq:APA} is a constrained optimization method that blends the current task and $b$ past tasks in the constraints, and it projects the previous model $\theta^{t-1}$ onto the constraint set. After task $t$ and before task $t+1$, \ref{eq:APA} discards $(x_{t-b},y_{t-b})$ and includes $(x_t,y_t)$ in its memory buffer.

When $b=0$, \ref{eq:APA} is equivalent to \ref{eq:LMS} with step size $\gamma=1$. Then, can the theoretical guarantees for \ref{eq:LMS} be extended for \ref{eq:APA} with an arbitrary $b$?  A visible challenge, though, is that the memory buffers at $b-1$ consecutive tasks contain repeated data, which means the consecutive updates of \ref{eq:APA} are ``correlated''. Overcoming this challenge requires extra technical devices; see, e.g., \cite[Problem IV.7]{Sayed-2008book}.

On the other hand, \ref{eq:APA} could have its own guarantees \textit{by design}. Indeed, if we use \textit{all} past data and set $b=t-1$, then \ref{eq:APA} guarantees to have zero losses on all seen tasks; this variant reads:
\begin{align}\label{eq:APA2}
	\theta^t_{\textnormal{APA}^\dagger} \in \argmin_{\theta \in \bbR^d} \| \theta - \theta^{t-1}\|_2^2 \quad \textnormal{s.t.} \quad y_{:t}= X_{:t}^\top \theta. \tag{APA$^\dagger$}
\end{align}
As $t$ grows, \ref{eq:APA2} changes its objective from $\| \theta - \theta^{t-1}\|_2^2$ to $\| \theta - \theta^{t}\|_2^2$. But, implicitly, there is no change in its objective and \ref{eq:APA2} consistently minimizes $\| \theta - \theta^0\|_2^2$ (with $\theta^0=0$) under its constraints:
\begin{theorem}\label{theorem:APA-obj-t->0}
	Suppose \cref{assumption:linear-unit-norm} holds. Let $\{\theta^t_{\textnormal{APA}^\dagger}\}_{t\geq 0}$ be the iterates of \ref{eq:APA2} produced with initialization $\theta^0_{\textnormal{APA}^\dagger}=0$. Then for all $t$ we have $\theta^t_{\textnormal{APA}^\dagger}$ equal to $\hat{\theta}^t$, where $\hat{\theta}^t$ is defined as
    \begin{align}\label{eq:min-norm}
		\hat{\theta}^t \in \argmin_{\theta \in \bbR^d} \| \theta \|_2^2 \quad \textnormal{s.t.} \quad y_{:t}= X_{:t}^\top \theta. \tag{Min-Norm}
	\end{align}
\end{theorem}
\begin{proof}
	Let $z^t:=\theta^t_{\textnormal{APA}^\dagger} - \theta^*$ and $\hat{z}^t:=\hat{\theta}^t - \theta^*$. Similarly to \cref{eq:kaczmarz-update2},  the output $\theta^t_{\textnormal{APA}^\dagger}$ of \ref{eq:APA2} satisfies
	\begin{align*}
		\theta^t_{\textnormal{APA}^\dagger} - \theta^* = P_t \left(\theta^{t-1}_{\textnormal{APA}^\dagger} - \theta^* \right) \Leftrightarrow z^t = P_t z^{t-1},
	\end{align*}
	where $P_t$ is the orthogonal projection onto $\textnormal{Span}(X_{:t})^\perp$ defined in \cref{eq:proj-P}. Similarly we can prove $\hat{z}^t = P_t z^0$. By definition one verifies $P_i X_i=0$ and then $P_iP_{i-1}=P_i$ for every $i$. Then we have $z^t = \prod_{i=1}^{t} P_i z^{0} = P_t z^0 = \hat{z}^t$. 
	This finishes the proof. See also \cref{fig:proj1+proj2=proj2} for a pictorial proof.	
	\tikzset{
		big arrow/.style={
			decoration={markings,mark=at position 1 with {\arrow[scale=0.9,#1]{>}}},
			postaction={decorate},
			shorten >=0.3pt},
		big arrow/.default=blue}
	\begin{figure}
		\centering
		\tdplotsetmaincoords{70}{85}
		
		\begin{tikzpicture}[tdplot_main_coords]
			\draw[thick, dashed] (-2,-1.414,1.414) -- (-2,0,0) -- (2,0,0) -- (2,-1.414,1.414) -- cycle;
			\node[black] at (-3,-2.1,0.4) {\footnotesize{$\textnormal{Span}(x_2)^\perp$}};
			
			\draw[thick] (-2,-1,0) -- (-2,2,0) -- (2,2,0) -- (2,-1,0) -- cycle;
			\draw[thick,mypurple](-2,0,0)--(2,0,0);
			
			
			\fill[cvprblue] (0,1.0605,1.0605) circle (1.5pt) node[above] {$z^0$};
			
			\draw[dashed, cvprblue, -{Stealth[length=2mm]}] (0,1.0605,1.0605) -- (0,0,0);
			
			\draw[dashed, myred, -{Stealth[length=2mm]}] (0,1.0605,1.0605) -- (0,1.0605,0) node[right] {$z^1$};
			
			\draw[dashed, myred, -{Stealth[length=2mm]}]  (0,1.0605,0) -- (0,0,0);
			\node[myred] at (0,0.32, -0.25) {$z^2$};

			\node[black] at (0,3,0)  {\footnotesize{$\textnormal{Span}(x_1)^\perp$}};
			
			\node[mypurple] (line) at (0,1,2) {\footnotesize{$\textnormal{Span}(x_1,x_2)^\perp$}};
			\draw[-latex, -{Stealth[length=2mm]}] (line) to[out=225,in=45] (-1,0,0);
			
			
			
			
		\end{tikzpicture}
		\caption{Pictorial proof of \cref{theorem:APA-obj-t->0} in $\bbR^3$. \ref{eq:APA2} projects $z^0$ onto $\textnormal{Span}(x_1)^\perp$ and then $\textnormal{Span}(x_1,x_2)^\perp$ (red arrows). \ref{eq:min-norm} projects $z^0$ directly onto  $\textnormal{Span}(x_1,x_2)^\perp$ (blue arrow). They reach the same point. \label{fig:proj1+proj2=proj2} }
	\end{figure}
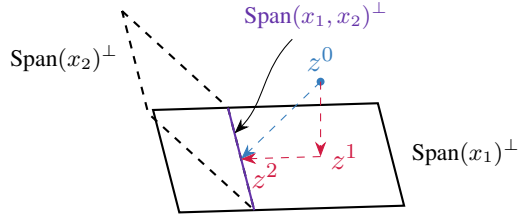
\end{proof}
What \cref{theorem:APA-obj-t->0} shows might be called the \textit{implicit bias} of \ref{eq:APA2}, as it implicitly finds the solution to another program, \ref{eq:min-norm}. An early result of this flavor is in \cite{Herman-C-ACM1978}, where it is shown that \ref{eq:LMS} converges to the solution $\hat{\theta}^t$ of \ref{eq:min-norm} for $t$-recurring tasks. Next, we examine a few recent continual learning methods, and we prove they all converge to $\hat{\theta}^t$, thereby establishing their equivalence to \ref{eq:APA2}. 

\myparagraph{ICL} The ICL framework of \cite{Peng-ICML2023} gives a general instruction: Minimize the loss of the current task, under the constraints that previous tasks are solved optimally. In our setting, this instruction means
\begin{align}\label{eq:ICL-AF}
	\min_{\theta \in \bbR^d} \left(y_t - x_t^\top \theta \right)^2 \quad \textnormal{s.t.} \quad y_{:t-1}= X_{:t-1}^\top \theta. \tag{ICL}
\end{align}
\ref{eq:ICL-AF} prioritizes remembering the past over solving the current task. To solve \ref{eq:ICL-AF} continually, we maintain and update a solution $\theta^t_{\textnormal{ICL}}$ to \ref{eq:ICL-AF} and the projection $P_t$ onto $\textnormal{Span}(X_{:t})^\perp$ defined in \cref{eq:proj-P}: 
\begin{prop}[Proposition 5 of \cite{Peng-ICML2023}, Simplified]\label{prop:ICL-AF-update}
	Define $P_{0} := I_d$ and $\theta^0_{\textnormal{ICL}}:=0$. Also define the projected data $\overline{x}_t:= P_{t-1} x_t$. Under \cref{assumption:linear-unit-norm}, we can update a solution $\theta^{t-1}_{\textnormal{ICL}}$ to \ref{eq:ICL-AF} via ($\forall t\geq 1$)
	\begin{align}\label{eq:ICL-AF-update}
		\theta^t_{\textnormal{ICL}} = \theta^{t-1}_{\textnormal{ICL}} - \frac{ x_t^\top \theta^{t-1}_{\textnormal{ICL}} - y_t  }{ \| \overline{x}_t\|_2^2 }  \cdot \overline{x}_t .
	\end{align}
\end{prop}
\begin{proof}
	For $t=1$, \ref{eq:ICL-AF} is unconstrained and requires minimizing $ \left(y_1 - x_1^\top \theta \right)^2$. Setting $\theta^1_{\textnormal{ICL}} = y_1 \cdot \frac{x_1}{\| x_1 \|_2^2 }$ minimizes it and gives  \cref{eq:ICL-AF-update}. For $t>1$, we are given $P_{t-1}$ and $\theta^{t-1}_{\textnormal{ICL}}$. Since every feasible point $\theta$ of \ref{eq:ICL-AF} can be written as $\theta^{t-1}_{\textnormal{ICL}} + P_{t-1}a$ for some $a\in \bbR^d$, \ref{eq:ICL-AF} is equivalent to
	\begin{align*}
		\min_{a \in \bbR^d} \left(y_t - x_t^\top (\theta^{t-1}_{\textnormal{ICL}} + P_{t-1}a) \right)^2 \Leftrightarrow \min_{a \in \bbR^d} \left(\overline{y}_t - \overline{x}_t^\top a \right)^2,
	\end{align*}
	where $\overline{y}_t:=y_t - x_t^\top \theta^{t-1}_{\textnormal{ICL}}$.	Setting $a = \overline{y}_t \cdot \frac{\overline{x}_t}{ \| \overline{x}_t\|_2^2 }$  minimizes the objective and gives \cref{eq:ICL-AF-update}.
\end{proof}
\begin{remark}[Update $P_t$]\label{remark:proj-P}
	Let $U_t$ be a matrix so that $U_{t}^\top U_{t}=I_d$ and the columns of  $X_{:t}$ and  $U_{t}$ span the same subspace. Then $P_t = I_d - U_t U_t^\top$. Thus, to update $P_t$ it suffices to update $U_t$. We can update $U_t$ via \textit{Gram-Schmidt Orthogonalization}: If $t=1$, we set $U_{1}=x_1$; if $t>1$, we obtain $U_{t}$ by attaching to $U_{t-1}$ a new column vector $\tilde{x}_t:= \frac{P_{t-1} x_t}{\| P_{t-1} x_t \|_2}$. 
    Alternatively, we can update $U_t$ via  Incremental \textit{Singular Value Decomposition} (SVD) \cite{Peng-arXiv2024}. Finally, we can also update $P_t$ directly (this will be useful for the sequel):
	\begin{align}\label{eq:update-Pt}
		P_t = I_d - U_t U_t^\top = I_d - U_{t-1} U_{t-1}^\top - \tilde{x}_t \tilde{x}_t^\top = P_{t-1} - \frac{P_{t-1} x_t x_t^\top P_{t-1}}{x_t^\top P_{t-1} x_t }.
	\end{align}
\end{remark}
\begin{remark}
	Division by $\| \overline{x}_t\|_2^2 $ in \cref{eq:ICL-AF-update} is valid as \cref{assumption:linear-unit-norm} ensures  $\overline{x}_t\neq 0$. The \ref{eq:LMS} update \cref{eq:LMS-SGD} is similar to \cref{eq:ICL-AF-update}, with a difference that \cref{eq:ICL-AF-update} uses the projected data $\overline{x}_t$ while \cref{eq:LMS-SGD} uses original data $x_t$.
\end{remark}

We can now show the implicit bias of the \ref{eq:ICL-AF} update \cref{eq:ICL-AF-update} towards the minimum-norm solution:
\begin{theorem}\label{theorem:ICL-ib}
	Suppose \cref{assumption:linear-unit-norm} holds.  Let $\{\theta^t_{\textnormal{ICL}^*}\}_{t\geq 0}$ be the iterates of \ref{eq:ICL-AF} as defined in \cref{prop:ICL-AF-update}. Then for all $t$ we have $\theta^t_{\textnormal{ICL}}$ equal to the solution $\hat{\theta}^t$ of \ref{eq:min-norm}.	
\end{theorem}
\begin{proof} 
    By construction, we have $y_{:t}= X_{:t}^\top \theta^t_{\textnormal{ICL}}$. Thus $\theta^t_{\textnormal{ICL}}= \hat{\theta}^t + P_t a$ for some $a\in \bbR^d$. Since $P_tX_{:t}=0$ and $\hat{\theta}^t\in \textnormal{Span}(X_{:t})$, it suffices to prove $\theta^t_{\textnormal{ICL}}\in \textnormal{Span}(X_{:t})$, as this implies $a=0$. 
    
    This is true for $t=1$, as $\theta^1_{\textnormal{ICL}} = y_1 \cdot \frac{x_1}{\| x_1 \|_2^2 }\in \textnormal{Span}(x_1)$. Assume, inductively, that $\theta^{t-1}_{\textnormal{ICL}}\in \textnormal{Span}(X_{:t-1})$. Since $\theta^t_{\textnormal{ICL}}$ is a linear combination of $\theta^{t-1}_{\textnormal{ICL}}$ and $\overline{x}_t$, it remains to prove $\overline{x}_t\in \textnormal{Span}(X_{:t})$. This is indeed the case, as $(I_d - P_{t-1})$  is the projection onto $\textnormal{Span}(X_{:t-1})$ and $\overline{x}_t=x_t - (I_d - P_{t-1}) x_t$.
\end{proof} 
While \cref{theorem:ICL-ib,theorem:APA-obj-t->0} imply the equivalence between \ref{eq:ICL-AF} and \ref{eq:APA2}, the connection between \ref{eq:ICL-AF} and \ref{eq:APA} is less direct, as the latter formulation uses a memory buffer rather than all data. That said, a memory buffer version of \ref{eq:ICL-AF} can be derived, which will be equivalent to \ref{eq:APA}. Finally, even though the formulation of \ref{eq:ICL-AF} requires all data, the actual implementation in \cref{eq:ICL-AF-update} does not, and it needs only to maintain a projection matrix $P_t$; the algorithms we explore later all have this feature.

\myparagraph{Gradient Projection for Linear Models} The methods of \textit{Orthogonal Gradient Descent} (OGD) \cite{Farajtabar-AISTATS2020} and  \textit{Orthogonal Recursive Fitting} (ORFit) \cite{Min-CDC2022} specialized for linear models are given as: 
\begin{align}
	\theta^t &= \theta^{t-1} -  \gamma_t  \cdot P_{t-1} \cdot x_t \big(x_t^\top \theta^{t-1} - y_t \big), \label{eq:OGD-AF} \tag{OGD} \\ 
	\theta^t &= \theta^{t-1} -  \frac{1}{x_t^\top P_{t-1} x_t}  \cdot P_{t-1} \cdot x_t \big(x_t^\top \theta^{t-1} - y_t \big). \label{eq:ORFit} \tag{ORFit}
\end{align}
In \ref{eq:OGD-AF}, $\gamma_t $ is the stepsize, $x_t (x_t^\top \theta^{t-1} - y_t)$ is a gradient of $\left(y_t - x_t^\top \theta \right)^2 / 2 $, and $P_{t-1}$ projects the gradient onto $\textnormal{Span}(X_{:t})^\perp$.  \ref{eq:OGD-AF} entails no forgetting as it ensures  $\theta^t$ and $\theta^{t-1}$ to make the same prediction about a previous input $x_i$: we have $x_i^\top P_{t-1}= 0$, so $x_i^\top \theta^t = x_i^\top  \theta^{t-1}$. 
While this reasoning holds for any stepsize $\gamma_t$, a properly chosen $\gamma_t$ would solve task $t$ better. Indeed, \ref{eq:ORFit} sets $\gamma_t=1/ x_t^\top P_{t-1} x_t$ as this directly decreases the current loss to zero; namely, $x_t^\top \theta^t- y_t=0$. To \textit{prove} it, verify that \ref{eq:ORFit} is identical to the \ref{eq:ICL-AF} update \cref{eq:ICL-AF-update}. We have thus established the connections between \ref{eq:ORFit}, \ref{eq:ICL-AF}, \ref{eq:APA2}, \ref{eq:min-norm}.

\myparagraph{Gradient Projection for Deep Networks} We now extend the philosophy of the above methods for  \textit{deep networks} in a layer-wise manner, which gives the \textit{gradient projection methods} (GP)  \cite{Saha-ICLR2021,Wang-CVPR2021}. Here, we define a deep network of layer $L$ to be a function $f_{\theta}: \bbR^{d_0} \to \bbR^{d_L}$ parameterized by $\theta:=(\Theta_1,\dots,\Theta_L)$ with $\Theta_\ell$ being a $d_{\ell-1}\times d_\ell$ matrix. Moreover, $f_{\theta}$ transforms its input $x$ via the following rule ($f^0_{\theta}(x):=x$):
\begin{align*}
	f_{\theta}(x):= f^L_{\theta}(x), \quad  f^\ell_{\theta}(x) := \sigma \left( \left( f^{\ell-1}_{\theta}(x) \right)^\top \Theta_{\ell} \right), \ \forall \ell=1,\dots,L.
\end{align*} 
Above, $\sigma:\bbR \to \bbR$ is nonlinear and often called the \textit{activation function}; we apply $\sigma$ to a given input vector in an entry-wise fashion. 
Thus, a deep network $f_\theta$ linearly transforms its input $x$ via $\Theta_\ell$ and then passes it through the nonlinearity $\sigma$,  repeatedly for $L$ times. The vector $f^\ell_{\theta}(x)\in \bbR^{d_\ell}$ is the intermediate output, evaluated at $\theta$ and $x$, and is often called the \textit{output feature} of the $\ell$-the layer.



For each $\ell$, denote by $P_t^{\ell}\in \bbR^{d_{\ell}\times d_{\ell}}$ the orthogonal projection onto $\textnormal{Span}\big(f^{\ell}_{\theta^1}(x_1),f^{\ell}_{\theta^2}(x_2),\dots,f^{\ell}_{\theta^t}(x_t)\big)^\perp$. In words, $P_t^\ell$ is the projection onto the orthogonal complement of the subspace spanned by the output features of the $\ell$-th layer, which generalizes the definition of $P_t$ in the case of linear models \cref{eq:proj-P}. Note that the output feature $f^{\ell}_{\theta^i}(x_i)$ is evaluated at $\theta^i$, not at the newest model $\theta^t$. Indeed, after task $t$, we might have no access to previous data $x_i$ and could not calculate $f^{\ell}_{\theta^t}(x_i)$.

To update the previous model $\theta^{t-1}:=(\Theta_L^{t-1},\dots,\Theta_1^{t-1})$, GP takes the following formula:
\begin{align}\label{eq:GP}
	\Theta_\ell^t = \Theta_{\ell}^{t-1} - \gamma_t \cdot P_{t-1}^{\ell-1} \cdot \Delta_t^\ell, \ \forall \ell=1,\dots,L. \tag{GP}
\end{align}
Here, $\gamma_t$ denotes the stepsize, and $\Delta_t^\ell\in \bbR^{d_{\ell-1}\times d_\ell}$ denotes the \textit{update direction} that a standard deep learning algorithm would take (thus, \ref{eq:GP} would be a standard deep learning method if $P_{t-1}^{\ell-1}$ were the identity matrix). For simplicity, we think of $\Delta_t^\ell$ as the partial gradient of some loss function $L( f_{\theta}(x_t) ,y_t )$ with respect to $\Theta_\ell$ evaluated at $\theta=\theta^{t-1}$; and a typical loss function is $L(\hat{y},y)= \| \hat{y} - y\|_2^2$. Then:

\begin{theorem}[Lemma 1 \cite{Wang-CVPR2021}]\label{theorem:GP}
    Let $\{\theta^t\}_{t\geq 0}$ be the iterates of \ref{eq:GP}. Fix $i,t$ with $1\leq i < t$. Then, for the past sample $x_i$, the output features of any layer $\ell$ evaluated at $\theta^t$ and $\theta^i$ are the same: namely $f^\ell_{\theta^t} (x_i) = f^\ell_{\theta^i} (x_i)$. In particular, the output features of the final layer are the same: $f_{\theta^t}(x_i)=f_{\theta^i}(x_i)$.
    
\end{theorem}
\begin{proof}
    Note that $f^\ell_{\theta^t} (x_i) = f^\ell_{\theta^i} (x_i)$ for $\ell=0$. To do induction on $\ell$, assume $f_{\theta^{t}}^{\ell-1} (x_i)  = f_{\theta^{i}}^{\ell-1} (x_i)$. 
	Then we have $P_{t-1}^{\ell-1} f_{\theta^{i}}^{\ell-1} (x_i)=0$. The inductive hypothesis (i.h.) and \ref{eq:GP} imply
	\begin{align*}
		\left(f^{\ell-1}_{\theta^t}(x_i) \right)^\top \Theta_{\ell}^t \overset{\textnormal{i.h.}}{=} \left(f^{\ell-1}_{\theta^{i}}(x_i) \right)^\top \Theta_{\ell}^t \overset{\textnormal{\ref{eq:GP}}}{=} \left(f^{\ell-1}_{\theta^{i}}(x_i) \right)^\top \Theta_{\ell}^{t-1}.
	\end{align*}
	Feeding them to the nonlinear function $\sigma$ shows  $f_{\theta^{t}}^{\ell} (x_i)  = f_{\theta^{t-1}}^{\ell} (x_i)$. The proof is complete. See \cref{fig:GP}.
\end{proof}
\cref{theorem:GP} shows \ref{eq:GP} does not forget the past tasks. To show \ref{eq:GP} also learns the current task, we assume $P_{t-1}^{\ell-1} \Delta_t^\ell \neq 0$ and that partial gradient $\Delta_t^\ell$ is a \textit{descent direction}, meaning that the standard update $\Theta_\ell^t = \Theta_{\ell}^{t-1} - \gamma_t  \Delta_t^\ell$ with $\gamma_t$ small enough would decrease the loss values. Then note that the projected direction $P_{t-1}^{\ell-1} \Delta_t^\ell$ of \ref{eq:GP} points towards roughly the same direction as $\Delta_t^\ell$ (as their \textit{inner product} is positive), so it is also a descent direction; for a formal argument, use the \textit{Taylor expansion} or see \cite{Wang-CVPR2021}.

That said, we often have $P_{t-1}^{\ell-1}=0$ as the output feature matrix $[f^{\ell-1}_{\theta^1}(x_1),\dots,f^{\ell-1}_{\theta^{t-1}}(x_{t-1})]$ often has full row rank $d_{\ell-1}$ for sufficiently large $t$; in this case we have $P_{t-1}^{\ell-1} \Delta_t^\ell =0$ and \ref{eq:GP} becomes $\Theta_\ell^t = \Theta_{\ell}^{t-1}$, and no parameters will be updated. This issue is often addressed by replacing the feature matrix with its \textit{low-rank} approximation, and the resulting projection matrix (which replaces $P_{t-1}^{\ell-1}$) can also be updated incrementally (e.g., \cite{Peng-arXiv2024}). However, \cref{theorem:GP} no longer holds for this low-rank version of \ref{eq:GP}, depicting a basic trade-off between this practical variant and theoretical non-forgetting guarantees. 

\begin{figure}
	\centering
	\begin{tikzpicture}
		\def\squareSize{4}
		\def\subSquareSize{1}

		\draw[thick] (0, 1) rectangle (4, 4);

		\foreach \y in {3} {
			\foreach \x in {0,1,2,3} {
				\fill[cvprblue] (\x*\subSquareSize, \y*\subSquareSize) rectangle ++(\subSquareSize, \subSquareSize);
			}
		}
		\foreach \y in {2} {
			\foreach \x in {1,2,3} {
				\fill[mypurple] (\x*\subSquareSize, \y*\subSquareSize) rectangle ++(\subSquareSize, \subSquareSize);
			}
		}
		\foreach \y in {1} {
			\foreach \x in {2,3} {
				\fill[myred] (\x*\subSquareSize, \y*\subSquareSize) rectangle ++(\subSquareSize, \subSquareSize);
			}
		}

		\node[black,anchor=east,cvprblue] at (10.7,3.5)  {$f_{\theta^{1}}^{\ell} (x_1)=f_{\theta^{2}}^{\ell} (x_1)=f_{\theta^{3}}^{\ell} (x_1) =f_{\theta^{4}}^{\ell} (x_1)$};
		\node[black,anchor=east,mypurple] at (10.7,2.5)  {$f_{\theta^{2}}^{\ell} (x_2)= f_{\theta^{3}}^{\ell} (x_2) =f_{\theta^{4}}^{\ell} (x_2)$};
		\node[black,anchor=east,myred] at (10.7,1.5)  {$f_{\theta^{3}}^{\ell} (x_3) =f_{\theta^{4}}^{\ell} (x_3)$};
		
		\foreach \x in {1,2,3,4} {
			\node[black] at (\x - 0.5,4.8)  {$\theta^{\x}$};
			\draw[-{Stealth[length=2mm]}]  (\x - 0.5, 4.6) -- (\x - 0.5, 4.1);
		}
		
		\foreach \x in {1,2,3} {
			\node[black] at (-0.8, 4.5-\x) {$x_{\x}$};
			\draw[-{Stealth[length=2mm]}]  (-0.6, 4.5-\x) -- (-0.1, 4.5-\x);
		}
%
	\end{tikzpicture}
	\caption{\ref{eq:GP} ensures that, for any layer $\ell$ and any $t>i$, the output features $f_{\theta^{t}}^{\ell} (x_i)$ is invariant, even though $\theta^t$ changes with $t$. This is indicated by the same color in each row and proved in \cref{theorem:GP}.\label{fig:GP} }
\end{figure}
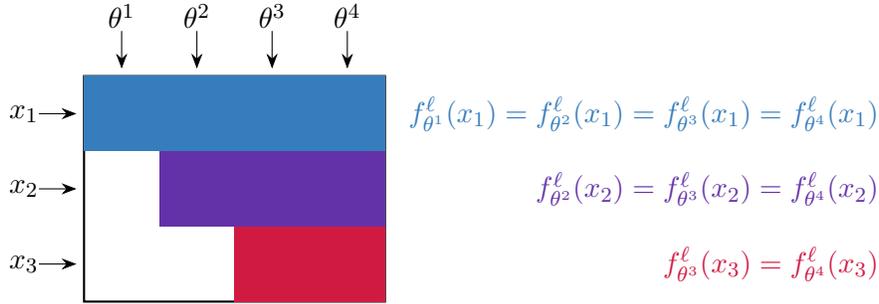

\myparagraph{Nonlinear Models and Linearization} We now consider extensions for nonlinear models by means of \textit{linearization}. Assume $y_t=f_t(\theta^*)$ for some nonlinear function $f_t$. 
To find $\theta^*$, we might consider a general form of \ref{eq:ICL-AF} to minimize $\left(y_t - f_t(\theta) \right)^2$ subject to constraints $ y_i= f_i(\theta)$ ($\forall i<t$). However, this program is difficult to solve, let alone continually. Instead, we could resort to the \textit{first-order Taylor expansion} $\tilde{f}_{i}$ of $f_i$ around the previously learned model $\theta^{i-1}$:
\begin{align}\label{eq:linearize-f}
	\tilde{f}_{i}(\theta):=f_i(\theta^{i-1})+ \nabla f_i(\theta^{i-1})^\top(\theta - \theta^{i-1}). 
\end{align}
Observe that we have $\tilde{f}_{i}(\theta)\approx f_i(\theta)$ for any $\theta$ that is close enough to $\theta^{i-1}$, while if $f_i$ is linear with $f_i(\theta)=x_i^\top \theta$ then we precisely have $\tilde{f}_{i}(\theta)= f_i(\theta)$ for any $\theta$. With $\tilde{f}_{i}$ we can then consider:
\begin{align*}
	\min_{\theta \in \bbR^d} \left(y_t - \tilde{f}_t(\theta) \right)^2 \quad \textnormal{s.t.} \quad y_i= \tilde{f}_i(\theta),\ i=1,\dots,t-1.
\end{align*}
This linearized problem is now almost identical to \ref{eq:ICL-AF} and can be solved similarly (cf. \cref{prop:ICL-AF-update}). Also we can extend \ref{eq:APA2} and \ref{eq:min-norm} for the nonlinear case here. Empirically, this idea of linearization, combined with engineering efforts, works well for deep networks in some settings \cite{Achille-CVPR2021}. Theoretically, we can prove results similar to \cref{theorem:APA-obj-t->0,theorem:ICL-ib} for the linearized problems (cf. \cite[Theorem 5]{Min-CDC2022}). On the other hand, deriving such guarantees for the original nonlinear problems remains challenging.

\myparagraph{Summary} For linear models, \ref{eq:APA2}, \ref{eq:ICL-AF}, and \ref{eq:ORFit} all converge to the same solution (of \ref{eq:min-norm}). These methods are extended for deep networks and nonlinear models, with non-forgetting guarantees (\cref{theorem:GP}). The setting of a single sample per task allows \ref{eq:ORFit} to find an optimal stepsize for \ref{eq:OGD-AF}, while no such stepsize is available in the case of multiple samples per task. That said, \ref{eq:ICL-AF} can be extended for the latter case. Our derivations in the section require all tasks to share a global minimizer (that is $\theta^*$), so that the constraints are feasible; we next explore methods that do not impose this assumption.

\section{Recursive Least-Squares (RLS)}
\myparagraph{Notations and Setup} We need the following version of the \textit{Woodbury matrix identity}:
\begin{align}\label{eq:woodbury}
	(\Phi + XB^{-1} X^\top)^{-1} = \Phi^{-1} - \Phi^{-1}X ( B + X^\top \Phi^{-1} X  )^{-1}  X^\top  \Phi^{-1} .
\end{align}
Above, $\Phi, X, B$ are matrices of compatible sizes with $\Phi,B$ positive definite. This formula can be verified by showing that the right-hand side of \cref{eq:woodbury} multiplying $\Phi + XB^{-1} X^\top$ gives the identity matrix. 

\myparagraph{RLS} Here we consider the \textit{exponentially-weighted} version of RLS:
\begin{align}\label{eq:RLS-AF}
	\theta^t\in \argmin_{\theta \in \bbR^d} \lambda \cdot \| \theta \|_2^2 + \sum_{i=1}^{t} \frac{\left( y_i - x_i^\top \theta \right)^2}{\beta^i}. \tag{RLS-($\beta,\lambda$)}
\end{align}
\ref{eq:RLS-AF}  has two positive hyperparameters, $\beta>0$ and $\lambda>0$. The residual $( y_i - x_i^\top \theta )^2$ is weighted by $1/\beta^i$. The regularization term $\lambda\cdot \| \theta \|_2^2$ 
ensures that \ref{eq:RLS-AF} has a unique solution.
\begin{remark}[RLS  Versus Continual Learning]\label{remark:RLS-CL}
	In adaptive filtering, $\beta$ is typically restricted to lie in $(0,1]$ and is called the \textit{forgetting factor}. In this case, $1/\beta^i$ increases with $i$, thus \ref{eq:RLS-AF} gradually downweights and eventually forgets the past samples. To learn continually without forgetting, one might set
	$\beta\geq 1$ and emphasize the past more. We will consider the case $\beta>0$ and the limit case $\beta\to 0$.
\end{remark}

We now derive an update formula of $\theta^t$ for \ref{eq:RLS-AF} given a new sample $(x_t,y_t)$. Let $B_t$ be the $t\times t$ diagonal matrix with its $i$-th diagonal being $\beta^i$, that is $B_t:=\diag(\beta,\beta^2,\dots,\beta^t)$. Define $\Phi_0:= I_d / \lambda$ and
\begin{align}\label{eq:def-Phi}
	\Phi_t:= \left(  \lambda \cdot I_d + X_{:t} B_t^{-1} X_{:t}^\top \right)^{-1}, \ \forall t \geq 1.  
\end{align}
The $d\times d$ matrix $\Phi_t$ can be viewed as the inverse of the \textit{Hessian matrix} of \ref{eq:RLS-AF}. We have: 
\begin{prop}\label{prop:RLS-AF-update}
	With $\theta^0:=0$, we can update the solution $\theta^{t-1}$ to \ref{eq:RLS-AF} by ($\forall t\geq 1$)
	\begin{align}\label{eq:RLS-update-theta}
		\theta^t = \theta^{t-1} - \frac{1}{ \beta^t + x_t^\top \Phi_{t-1} x_t}    \cdot \Phi_{t-1} \cdot x_t 
		\left(x_t^\top \theta^{t-1} - y_t \right).
	\end{align}
\end{prop}
\begin{remark}[Update $\Phi_t$]\label{remark:Phi-update}
	Since $\Phi_t^{-1}=\Phi_{t-1}^{-1} + x_t x_t^\top / \beta^t$, applying the Woodbury matrix identity \cref{eq:woodbury} gives
	\begin{align}\label{eq:RLS-update-Phi}
		\Phi_t= \Phi_{t-1} - \frac{\Phi_{t-1} x_t x_t^\top \Phi_{t-1} }{\beta^t + x_t^\top \Phi_{t-1} x_t}.
	\end{align}
\end{remark}
\begin{proof}
	Note that $\theta^t=\Phi_t \sum_{i=1}^{t} y_i \cdot x_i / \beta^i$ and $\theta^{t-1}=\Phi_{t-1} \sum_{i=1}^{t-1} y_i \cdot x_i / \beta^i$. It then follows from \cref{eq:RLS-update-Phi} that
	\begin{align*}
		\theta^t &= \left( \Phi_{t-1} - \frac{\Phi_{t-1} x_t x_t^\top \Phi_{t-1} }{\beta^t + x_t^\top \Phi_{t-1} x_t} \right) \sum_{i=1}^{t} y_i \cdot x_i / \beta^i \\ 
		&=\left( I_d - \frac{\Phi_{t-1} x_t x_t^\top  }{\beta^t + x_t^\top \Phi_{t-1} x_t} \right) (\theta_{t-1} + y_t \cdot \Phi_{t-1} x_t / \beta^t ).
	\end{align*}
	Rearranging the terms gives \cref{eq:RLS-update-theta} and therefore finishes the proof.
\end{proof}
By definition $\Phi_t$ is positive definite, hence, from \cref{eq:RLS-update-Phi}, we see a \textit{decreasing} sequence of positive definite matrices, $\Phi_0\succeq \Phi_1 \succeq \cdots \succeq \Phi_t \succ 0$. This implies $\Phi_t$ eventually converges and so does \ref{eq:RLS-AF}.

\myparagraph{RLS and Orthogonal Projection} Using the Woodbury matrix identity \cref{eq:woodbury}, rewrite  $\Phi_t$ as
\begin{align*}
	\Phi_t=  \frac{1}{\lambda} \bigg( I_d - X_{:t} \Big( \lambda \cdot B_t +  X_{:t}^\top X_{:t} \Big)^{-1} X_{:t}^\top \bigg). 
\end{align*}
In the limit $\beta\to 0$ (and thus $B_t\to 0$), \hyperref[eq:RLS-AF]{\textcolor{cvprblue}{RLS-($0,\lambda$)}} would put infinitely large weights on the error terms $(y_i - x_i^\top \theta)^2$, and become equivalent to minimizing $ \lambda \cdot \| \theta \|_2^2$ subject to constraints $y_{:t}=X_{:t}^\top \theta$, that is equivalent to \ref{eq:min-norm} up to scale $\lambda$. This is a hint to the following algorithmic equivalence:
\begin{lemma}\label{lemma:RLS=ICL}
	Let $P_{0} := I_d$ and $\{ \theta^t_{\textnormal{ICL}} \}_{t\geq 0}$ be defined in \cref{prop:ICL-AF-update}. Let $\{(\theta^t,\Phi_t)\}_{t\geq 0}$ be the iterates of \ref{eq:RLS-AF} as defined in \cref{prop:RLS-AF-update,remark:Phi-update}. If $\beta = 0$, then $\theta^t_{\textnormal{ICL}}=\theta^t$ and $P_t= \lambda \Phi_t$.
\end{lemma}
\begin{proof}
	Since $P_0= \lambda \Phi_0$ and their updates in  \cref{eq:update-Pt} and \cref{eq:RLS-update-Phi} are identical, we get $P_t= \lambda \Phi_t$ for all $t$. Since $\theta^0_{\textnormal{ICL}}=\theta^0$ and their updates in \cref{eq:ICL-AF-update} and \cref{eq:RLS-update-theta} are also identical, we get $\theta^t_{\textnormal{ICL}}=\theta^t$ for all $t$ as well.
\end{proof}
Thus, \hyperref[eq:RLS-AF]{\textcolor{cvprblue}{RLS-($0,\lambda$)}} is equivalent to the aforementioned \ref{eq:APA2}-type methods, and \ref{eq:RLS-AF} generalize them. The derivation of \ref{eq:RLS-AF} does not assume a shared solution $\theta^*$  among all tasks  (\cref{assumption:linear-unit-norm}); when such $\theta^*$ does not exist, it computes a (weighted) \textit{average} among the solutions of all tasks (\cref{fig:RLS}). 


\tikzset{
	big arrow/.style={
		decoration={markings,mark=at position 1 with {\arrow[scale=0.9,#1]{>}}},
		postaction={decorate},
		shorten >=0.3pt},
	big arrow/.default=blue}
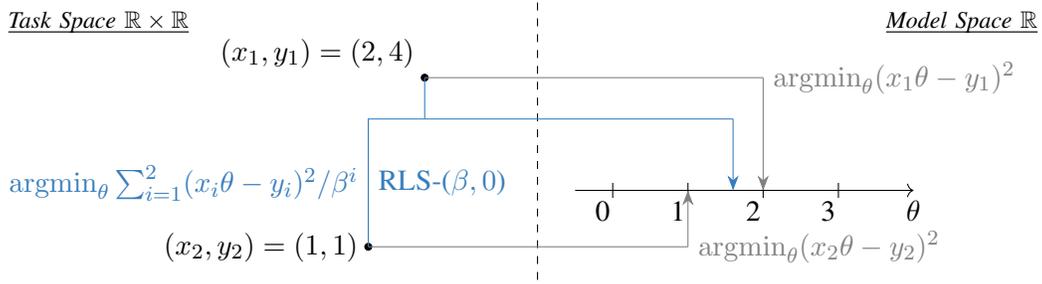
\begin{figure}
	\centering
	\begin{tikzpicture}
		\node[black,anchor=east] at (-4.5,2)  {\appallingunderline{\footnotesize{\textit{Task Space} $\bbR\times \bbR$}} };
		
		\fill[black] (-1.5,1.25) circle (1.5pt) node[above left] {$(x_1,y_1)=(2,4)$};
		\fill[black] (-2.25,-1) circle (1.5pt) node[left] {$(x_2,y_2)=(1,1)$};
		
		\draw[dashed] (0,2.25) -- (0,-1.5);
		
		\node[black,anchor=west] at (4.5,2)  {\appallingunderline{\footnotesize{\textit{Model Space} $\bbR$}} };
		
		\draw[->] (0.5,-0.25) -- (5,-0.25) node[anchor=north] {$\theta$};
		
		\foreach \x in {0, 1, 2, 3} {
			\draw (\x+1, -0.15) -- (\x+1, -0.35);
			\node[black,anchor=east] at (\x+1.1, -0.525)  {\x};
		}
		
		\draw[gray] (-1.5,1.25) -- (3,1.25)  node[anchor=west] {$\argmin_\theta (x_1\theta - y_1)^2$};
		\draw[gray,-{Stealth[length=2mm]}] (3,1.25) -- (3,-0.25);
		
		\draw[gray] (-2.25,-1) -- (2,-1)  node[anchor=west] {$\argmin_\theta (x_2\theta - y_2)^2$};
		\draw[gray,-{Stealth[length=2mm]}] (2,-1) -- (2,-0.25);
		
		\foreach \x in {0.7} {
			\draw[cvprblue] (-2.25,-1) -- node[right] {\hyperref[eq:RLS-AF]{\textcolor{cvprblue}{RLS-($\beta,0$)}}} node[left] { $\argmin_\theta \sum_{i=1}^{2}(x_i\theta - y_i)^2/\beta^i$ } (-2.25,\x) -- (2.6,\x);
			\draw[cvprblue,-{Stealth[length=2mm]}] (2.6,\x) -- (2.6,-0.25) ;
			
			\draw[cvprblue] (-1.5,1.25) -- (-1.5,\x);
		}
		
		
	\end{tikzpicture}
	\caption{For the two tasks with data $(x_1,y_1)=(2,4)$ and $(x_2,y_2)=(1,1)$, the respective best models are $\theta_1^*=2$ and $\theta_2^*=1$ (gray arrows). \hyperref[eq:RLS-AF]{\textcolor{cvprblue}{RLS-($\beta,0$)}} outputs a model that averages $\theta_1^*$ and $\theta_2^*$ (blue arrows). \label{fig:RLS}}
\end{figure}


\myparagraph{Linearized RLS and Layer-wise RLS} From the equivalence in \cref{lemma:RLS=ICL}, it follows from the previous section that  \ref{eq:RLS-AF} applies to \textit{linearized} nonlinear models (linearized RLS), and to deep networks in a layer-wise fashion (layer-wise RLS). To compare the two methods, assume the network has $L$ layers and each layer has $D$ parameters. Then linearized RLS needs $O(L^2 D^2)$ memory to store the projection matrix (e.g., $\Phi_t$), while layer-wise RLS needs $O(LD^2)$ memory to store $L$ different $D\times D$ projection matrices, one for each linear layer. Time complexities of both methods can be compared similarly.

On a historical note, layer-wise RLS was used to accelerate training neural networks around 1990, and in survey  \cite{Shah-1992} it was called \textit{Enhanced Back Propagation} (EBP). The method of \textit{Orthogonal Weight Modification} (OWM) \cite{Zeng-NMI2019} can be viewed as an application of EBP to continual learning. While OWM empirically reduces catastrophic forgetting, it has not been mathematically clear how much (layer-wise) RLS forgets with arbitrary $\beta,\lambda$. Moreover, RLS and its layer-wise extension are prone to numerical errors. These issues are later alleviated by \ref{eq:GP}, as \ref{eq:GP} typically updates the orthogonal projection via numerically robust SVD and its non-forgetting property is proved for deep networks (\cref{theorem:GP}). 


\myparagraph{RLS for Dynamically Expanding Networks} Besides linearization and the layer-wise philosophy, the third use of RLS in deep continual learning involves pre-trained models that are now widely available. Pre-trained models provide generalizable features, which can be used for continual learning of downstream tasks. Pre-trained models also simplify network design, as strong performance can be achieved by simply combining them with shallow networks such as a linear classifier. Here we consider classification tasks with pre-trained models and a new task might contain previously unseen classes, a setting known as \textit{class-incremental learning}. To solve these tasks continually, we train a linear classifier using the output features of a given, frozen, pre-trained model. In response to the growing number of seen classes, the classifier needs to grow new neurons in order to make predictions about all seen classes (\cref{fig:dynamic-classifier}).


\begin{figure}
    \centering
    \includegraphics[width=0.85\textwidth]{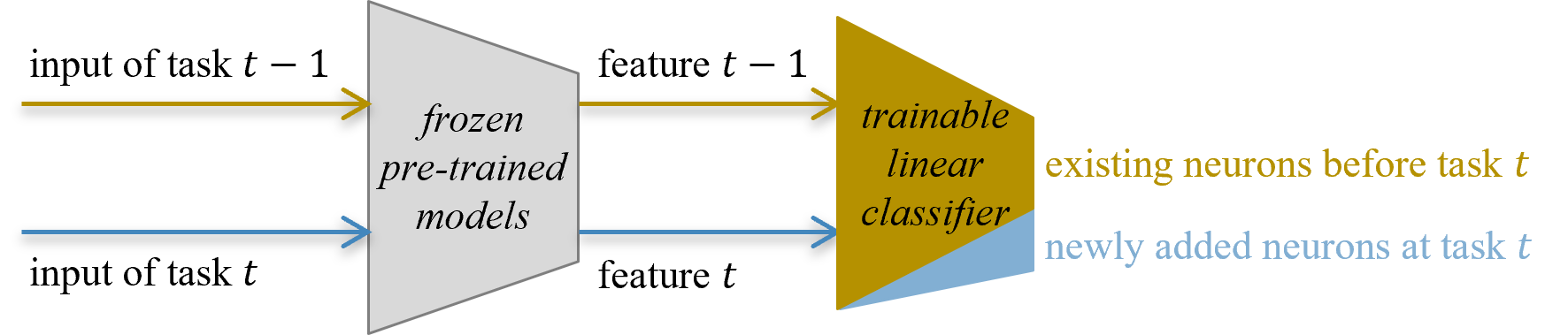}
    \caption{A dynamically expanding architecture for continual classification. If a new task contains new classes, then the classifier adds new neurons in order to make predictions for all seen classes. }
    \label{fig:dynamic-classifier}
\end{figure}

To formalize the situation, assume at task $t$ we are given the output feature $x_t\in\bbR^d$ from the pre-trained model and also a class label $y_t$. Here $y_t$ is some \textit{one-hot vector} of dimension $c_t$ (that is a standard basis vector in $\bbR^{c_t}$), where $c_t$ is the number of classes seen thus far. By definition, we have $c_1\leq \cdots \leq c_t$. If $y_t$ has the entry $1$ at position $k$, then $x_t$ is in class $k$. At task $t$, we have seen $c_t$ classes, but the previous label $y_i$ has dimension $c_i$; we need to pad $c_t-c_i$ zero entries to $y_i$ to match the dimension. For simplicity, we view $y_i$ as having $c_t$ entries (zeros are implicitly padded if needed). Our goal is to continually train a classifier $\Theta\in \bbR^{d\times c_t}$, which grows its columns as $c_t$ increases.


Similarly to \ref{eq:RLS-AF}, we consider solving ($\| \cdot \|_{\textnormal{F}}$ denotes the Frobenius norm of a matrix)
\begin{align}\label{eq:D-RLS-AF}
    \Theta^t\in \argmin_{ \Theta\in \bbR^{d\times c_t} } \lambda \cdot \| \Theta \|_{\textnormal{F}}^2 + \sum_{i=1}^{t} \frac{\left\| y_i - \Theta^\top x_i \right\|_2^2}{\beta^i}. \tag{D-RLS}
\end{align}
Since the continual update of $\Theta^{t-1}$ into $\Theta^t$ is similar to that of \cref{prop:RLS-AF-update}, we omit it here. We furthermore refer the reader to \cite{Zhuang-NeurIPS2022} for extra details and to \cite{Azimi-TNN1993} for a variant that expands the network layer-wise. A final remark is that the features $[x_1,\dots,x_t]$ from pre-trained models tend to be ill-conditioned, which might amplify the numerical instability of RLS; see \cite{Peng-arXiv2024} for a detailed account and a remedy.


\myparagraph{Summary} \ref{eq:RLS-AF} generalizes the \ref{eq:APA2} family, and when different tasks have different solutions, \ref{eq:RLS-AF} \textit{averages} these solutions (cf. \cref{fig:RLS}). However, if these solutions are far from each other, the average would not solve any of the tasks well. In the next section, we address this issue by imposing a precise mathematical relationship between the solution of each task and exploring its ramifications.

\section{Kalman Filter (KF)}
\myparagraph{Notations and Setup} We consider the \textit{Linear Gaussian Model}:
\begin{equation}\label{eq:LGM}
	\begin{split}
		\theta_i &= A_i \theta_{i-1} + w_i  \ \ (i\geq 2), \\ 
		y_i &= X_i^\top \theta_i + v_i.
	\end{split} \tag{LGM}
\end{equation} 
Here, $A_i,X_i$ are matrices of sizes $d\times d$ and $d\times m$, respectively, while $\theta_1$, all $w_i$'s, all $v_i$'s are independent and Gaussian; we assume their means and covariances are known and given by
\begin{align*}
	\theta_1\sim \cN(\mu_1, \Sigma_1),\ w_i\sim \cN(0, Q_i),\ v_i\sim \cN(0,R_i).
\end{align*}
Here, $\Sigma_1, Q_i$ are $d\times d$ positive definite matrices, $\mu_1\in \bbR^d$, and $R_i$ is positive definite of size $m\times m$. Different from the previous equation $y_i=x_i^\top \theta^*$, each task $i$ now has $m$ equations $y_i = X_i^\top \theta_i + v_i$, thus $m$ samples, and $y_i$ is a random vector of dimension $m$. One more difference is the extra equation $\theta_i = A_i \theta_{i-1} + w_i$, which means each task $i$ has its own true model $\theta_i$, and this equation defines a linear relationship between two consecutive tasks. Despite these differences, the goal is similar (vaguely put): estimate $\theta_1,\dots,\theta_t$ continually. 
\ref{eq:LGM} relies on Gaussian assumptions, and the derivation of this section requires exercising probabilistic reasoning. The first exercise is ($a|b$ denotes $a$ conditioned on $b$):
\begin{lemma}\label{lemma:ab+a|b}
	Assume $b=X^\top a + \omega$, where $a,\omega$ are independent and jointly Gaussian with $a\sim \cN(\mu_a, \Sigma_a)$, $\omega\sim \cN(0,\Sigma_{\omega})$, and $X$ is a fixed matrix. Then $a,b$ are jointly Gaussian and  $b$ has mean $X^\top \mu_a$ and covariance $\Sigma_b:=\Sigma_{\omega}  + X^\top \Sigma_{a} X$. Also,  $a|b$ is  Gaussian, with its mean $\mu_{a|b}$ and covariance $\Sigma_{a|b}$ given by
    \begin{equation}\label{eq:a|b}
        \begin{split}
            \mu_{a|b} &= \mu_a -   \Sigma_a X  \Sigma_{b}^{-1} (X^\top \mu_a - b), \\ 
            \Sigma_{a|b} &= \Sigma_a - \Sigma_a X   \Sigma_{b}^{-1} X^\top \Sigma_a .
        \end{split} 
    \end{equation}
\end{lemma}
\begin{proof}
    Recall that a random vector is Gaussian if and only if any linear combination of its entries is Gaussian. Since $a$ and $\omega$ are jointly Gaussian, any linear combination of their entries is Gaussian. Since $b=X^\top a + \omega$, any linear combination of entries of $b,a$ is a linear combination of entries of $a,\omega$, which is Gaussian. This implies $a,b$ are jointly Gaussian. We have $\bbE[b]=\bbE[X^\top a + \omega] = X^\top \mu_a$ and
	\begin{align*}
		\bbE[ (b-\bbE[b]) (b - \bbE[b])^\top ] &= \bbE[ (X^\top a + \omega-X^\top \mu_a) (X^\top a + \omega - X^\top \mu_a)^\top ],
	\end{align*}
	which is simplified to $\Sigma_{\omega}  + X^\top \Sigma_{a} X$. To prove $a|b$ is Gaussian and \cref{eq:a|b}, see, e.g., \cite[Section 2.3.1]{Bishop-2006}.
\end{proof}
In \cref{lemma:ab+a|b}, Of the continual learning flavor is the update of $\mu_a$ into $\mu_{a|b}$, which is upon observing $b$. Indeed, if $X$ is a vector and $b$ a scalar, then \cref{eq:a|b} can be written as
\begin{equation*}
	\begin{split}
		\mu_{a|b} &= \mu_a - \frac{1}{\Sigma_{\omega} + X^\top \Sigma_{a} X  } \cdot \Sigma_a \cdot X (X^\top \mu_a - b), \\ 
		\Sigma_{a|b} &= \Sigma_a -  \frac{\Sigma_a X  X^\top \Sigma_a}{\Sigma_{\omega} + X^\top \Sigma_{a} X }. 
	\end{split}
\end{equation*}
This matches the RLS update in \cref{prop:RLS-AF-update} and \cref{remark:Phi-update}. For this reason, \cref{eq:a|b} can be viewed as a generalization of the RLS update for the case of observing multiple samples per task. 

\myparagraph{MAP Formulation} We consider estimating $\theta_1,\dots,\theta_t$ via the \textit{Maximum A Posteriori} principle: 
\begin{align}\label{eq:map}
	\theta_{i|t} \in \argmax_{\theta_i\in \bbR^d} p\left(\theta_i \ |\ y_1,\dots, y_t \right). \tag{MAP}
\end{align}
Here, $p_i(\cdot\ |\ \cdot )$ is the conditional \textit{probability density function} (pdf); we allow $t=0$ and this defines $\theta_{i|0}$ as a maximizer of $p(\theta_i)$. 
Since $\theta_0,w_i,v_i$ in \ref{eq:LGM} are jointly Gaussian and $\theta_i$ and $y_i$ are linear functions of them, \cref{lemma:ab+a|b} indicates $\theta_i,w_i,v_i,y_i$ ($\forall i$) are jointly Gaussian. This furthermore implies $\theta_i\ |\ y_{:t}$ is Gaussian (here note that $y_{:t}:=[y_1^\top,\dots,y_t^\top]^\top$). As a basic property of the Gaussian pdf, $p\left(\theta_i \ |\ y_{:t} \right)$ is maximized at the mean $\bbE[\theta_i \ |\ y_{:t}]$. In other words, the \ref{eq:map} estimate $\theta_{i|t}$ is given by
\begin{align}
	\theta_{i|t} = \bbE[\theta_i \ |\ y_1,\dots,y_t].
\end{align}
We can therefore solve \ref{eq:map} by tracking the means $\theta_{i|t}$ as $i$ and $t$ vary. The classic \textit{Kalman Filter} (KF) does so continually as $t$ increases, to track the most recent mean $\bbE[\theta_t \ |\ y_{:t}]$ given all seen data $y_{:t}$. On the other hand, the \textit{Rauch-Tung-Striebel smoother} (RTS) initializes itself at $\theta_{t|t}$ and computes $\theta_{t-1|t}, \theta_{t-2|t},\dots,\theta_{1|t}$ in cascade and \textit{backward}. Combining KF and RTS yields a scheme that computes all $\theta_{i|t}$ with $i\leq t$; these are the upper triangular part of the matrix in \cref{fig:kf-rts}. What is missing in \cref{fig:kf-rts} is that  KF and RTS compute the mean $\theta_{i|t}$ together with the \textit{error covariacne} $\Sigma_{i|t}$:
\begin{align}\label{eq:Sigma-def}
	\Sigma_{i|t}:= \bbE\left[ (\theta_i - \theta_{i|t}) (\theta_i - \theta_{i|t})^\top\ |\ y_1,\dots,y_t \right].
\end{align}
As we will see in the subsequent derivations, $\Sigma_{i|t}$ facilitates transferring the estimate for one task to an estimate for another task, playing a role similar to the projection $P_t$ (or $\Phi_t$) of the previous section.

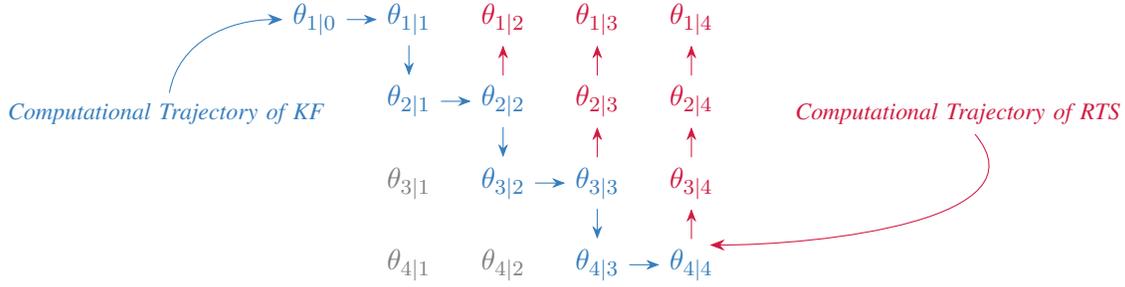
\begin{figure}
	\centering
	\begin{tikzpicture}
		\node[cvprblue] at (1, 4) (11) {$\theta_{1|1}$};
		\node[cvprblue, below=0.4cm of 11] (21) {$\theta_{2|1}$};
		\node[cvprblue, right=0.4cm of 21] (22) {$\theta_{2|2}$};
		\node[cvprblue, below=0.4cm of 22] (32) {$\theta_{3|2}$};
		\node[cvprblue, right=0.4cm of 32] (33) {$\theta_{3|3}$};
		\node[cvprblue, below=0.4cm of 33] (43) {$\theta_{4|3}$};
		\node[cvprblue, right=0.4cm of 43] (44) {$\theta_{4|4}$};

		\node[cvprblue, left=0.4cm of 11] (10) {$\theta_{1|0}$};
		\draw[cvprblue,-{Stealth[length=1.75mm]}] (10) -- (11);
		
		\draw[cvprblue,-{Stealth[length=1.75mm]}] (11) -- (21);
		\draw[cvprblue,-{Stealth[length=1.75mm]}] (21) -- (22);
		\draw[cvprblue,-{Stealth[length=1.75mm]}] (22) -- (32);
		\draw[cvprblue,-{Stealth[length=1.75mm]}] (32) -- (33);
		\draw[cvprblue,-{Stealth[length=1.75mm]}] (33) -- (43);
		\draw[cvprblue,-{Stealth[length=1.75mm]}] (43) -- (44);
		
		\node[myred, right=0.4cm of 11] (12) {$\theta_{1|2}$};
		\node[myred, right=0.4cm of 12] (13) {$\theta_{1|3}$};
		\node[myred, right=0.4cm of 13] (14) {$\theta_{1|4}$};
		
		\node[myred, right=0.4cm of 22] (23) {$\theta_{2|3}$};
		\node[myred, right=0.4cm of 23] (24) {$\theta_{2|4}$};
		
		\node[myred, right=0.4cm of 33] (34) {$\theta_{3|4}$};

		\draw[myred,-{Stealth[length=1.75mm]}] (22) -- (12);
		
		\draw[myred,-{Stealth[length=1.75mm]}] (33) -- (23);
		\draw[myred,-{Stealth[length=1.75mm]}] (23) -- (13);
		
		\draw[myred,-{Stealth[length=1.75mm]}] (44) -- (34);
		\draw[myred,-{Stealth[length=1.75mm]}] (34) -- (24);
		\draw[myred,-{Stealth[length=1.75mm]}] (24) -- (14);

		\node[gray, below=0.4cm of 21] (31) {$\theta_{3|1}$};
		\node[gray, below=0.4cm of 31] (41) {$\theta_{4|1}$};
		\node[gray, right=0.4cm of 41] (42) {$\theta_{4|2}$};

		\node[black,anchor=east] at (0,2.75) (kf) {\footnotesize{\textit{\textcolor{cvprblue}{Computational Trajectory of KF}}}};
		
		\node[black,anchor=west] at (6,2.75) (rts) {\footnotesize{\textit{\textcolor{myred}{Computational Trajectory of RTS}}}};
		
		\draw[-latex, -{Stealth[length=2mm]},cvprblue] (kf) to[out=80,in=180] (10);
		
		\draw[-latex, -{Stealth[length=2mm]},myred] (rts) to[out=-50,in=0] (5, 1);
	\end{tikzpicture}
	\caption{The \textit{Kalman Filter} (KF) follows blue arrows and computes  $\theta_{i|t}$ in order to track $\theta_{t|t}$. The \textit{Rauch-Tung-Striebel smoother} (RTS) begins with $\theta_{t|t}$ and computes $\theta_{t-1|t}, \theta_{t-2|t},\dots,\theta_{1|t}$ in succession.  \label{fig:kf-rts} }
\end{figure}


\myparagraph{The KF Recursion} We prove the classic KF recursion in the following.
\begin{prop}
	Define the ``Kalman gain'' matrix $K_t:=\Sigma_{t|t-1} X_t  (R_t  + X_t^\top \Sigma_{t|t-1} X_t)^{-1}\in\bbR^{d\times m}$. KF is initialized at $\theta_{1|0} = \bbE[\theta_1]= \mu_1$ and $\Sigma_{1|0} = \Sigma_1$. For $t\geq 1$, KF updates $\theta_{t|t-1}$ and the corresponding error covariance $\Sigma_{t|t-1}$ via (arrows \textcolor{cvprblue}{$\rightarrow$} in \cref{fig:kf-rts})
	\begin{equation}\label{eq:kf-correction}
		\begin{split}
			\theta_{t|t} &= \theta_{t|t-1} -   K_t (X_t^\top \theta_{t|t-1} - y_t), \\ 
			\Sigma_{t|t} &= \Sigma_{t|t-1}  - K_t X_t^\top \Sigma_{t|t-1}.
		\end{split}  \tag{KF$\rightarrow$}
	\end{equation}
	For $t\geq 1$, KF updates $\theta_{t|t}$ and the corresponding error covariance $\Sigma_{t|t}$ via (arrows \textcolor{cvprblue}{$\downarrow$} in \cref{fig:kf-rts})
	\begin{align}\label{eq:kf-pred}
		\theta_{t+1|t} = A_{t+1} \theta_{t|t}, \ \Sigma_{t+1|t} = Q_{t+1} + A_{t+1} \Sigma_{t|t} A_{t+1}^\top. \tag{KF$\downarrow$}
	\end{align}
\end{prop}
\begin{proof}
	\ref{eq:kf-correction} follows from \cref{lemma:ab+a|b} and for this we verify the conditions of \cref{lemma:ab+a|b}: $\theta_t| y_{:t-1}$ is Gaussian with mean $\theta_{t|t-1}$ and covariance $\Sigma_{t|t-1}$; $y_t|y_{:t-1}$ is a linear function of $\theta_t|y_{:t-1}$ and $v_t$ (recall \ref{eq:LGM}). \ref{eq:kf-pred} follows from \cref{lemma:ab+a|b} and the fact that $\theta_{t+1}|y_{:t}$ is a linear function of $\theta_t |y_{:t}$ and $w_{t+1}$.
\end{proof}
\ref{eq:kf-correction} is sometimes called the \textit{correction step} as it revises the estimate based on new observation $y_t$, and \ref{eq:kf-pred} is called the \textit{prediction step} as it computes $\theta_{t+1|t}$ without new observations.

%
%
%


\myparagraph{KF and RLS} 
The following lemma shows that \ref{eq:RLS-AF} is a special case of KF: 
\begin{lemma}\label{lemma:kf-rls}
	Let $\{(\theta^t,\Phi_t)\}_{t\geq 0}$ be the iterates of \ref{eq:RLS-AF} as defined in \cref{prop:RLS-AF-update,remark:Phi-update}. Let $\{ \theta_{t|t}, \Sigma_{t|t} \}_{t\geq 1}$ be the iterates of KF applied to the following special case of \ref{eq:LGM} ($i\geq 2$)
	\begin{align*}
		\theta_i = \theta_{i-1},\ \  y_i = x_i^\top \theta_i + v_i
	\end{align*}
	with $\theta_1\sim \cN(\theta^1,\Phi_1)$ and $v_i\sim(0, \beta^i)$. Then $\theta_{t|t}=\theta^t$ and $\Sigma_{t|t}= \Phi_t$ for all $t\geq 1$.	
\end{lemma}
\begin{proof}
	Using the assumptions to simplify the KF updates gives formulas identical to RLS udpates.
\end{proof}
\cref{lemma:kf-rls} reveals that \ref{eq:RLS-AF} implicitly relies on a shared solution to all tasks ($\theta_i = \theta_{i-1}$) and that it generalizes the \ref{eq:APA2} family by allowing noise $v_i$ added to measurements $y_i$.

\myparagraph{Extensions of KF} Similarly to \ref{eq:RLS-AF} and the \ref{eq:APA2} family, KF can be extended for deep networks (layer-wise), dynamically growing a linear classifier, and for nonlinear models. To illustrate the idea, we treat the case of nonlinear models. With nonlinear functions  $g_i, f_i$, consider the nonlinear version of \ref{eq:LGM}:
\begin{equation*}
	\begin{split}
		\theta_i &= g_i(\theta_{i-1}) + w_i  \ \ (i\geq 2), \\ 
		y_i &= f_i(\theta_i) + v_i, 
	\end{split}
\end{equation*}
Different from previous linearization \cref{eq:linearize-f}, we now have two functions $g_t,f_t$ to be linearized at each task $t$. Linearizing them around different points yields different variants. Similarly to \cref{eq:linearize-f}, the \textit{Extended Kalman Filter} (EKF) performs linearization around the \textit{most recent estimate}: It linearizes $g_{t+1}$ around $\theta_{t|t}$ and  $f_{t+1}$ around $\theta_{t+1|t}$. While EKF sets $\theta_{t+1|t}=g_{t+1}(\theta_{t|t})$, it computes other estimates by replacing the roles of $A_{t}, X_t$ in the KF updates (\ref{eq:kf-correction}, \ref{eq:kf-pred}) with the linearized functions.

\myparagraph{Positive Backward Transfer} A continual learning method is said to exhibit \textit{positive backward transfer}, if at the current task it \textit{improves}, not just \textit{maintains}, the performance on past tasks. This is a stronger property than not forgetting, but theoretical understanding of positive backward transfer is much more limited in the continual learning literature. The final topic of this tutorial is to illustrate that, in \ref{eq:LGM}, positive backward transfer follows naturally from combining KF and RTS (\cref{fig:kf-rts}).

Recall that, when $\theta_{t|t}$ is learned from task $t$, RTS computes $\theta_{i|t}$, an estimate for the $i$-th task ($i<t$). To claim this yields positive backward transfer, we need to prove $\theta_{i|t}$ is \textit{better} than any previous estimate $\theta_{i|s}$ ($i\leq s < t$).  Intuitively, $\theta_{i|t}$ might be better, as it uses  $(t-s)m$ more samples for estimation, but this is not absolutely clear, as the extra samples are from different tasks. Formally, we evaluate the quality of $\theta_{i|t}$ using the expected squared error $\bbE[\| \theta_i - \theta_{i|t} \|_2^2]$, that is the trace $\trace(\Sigma_{i|t})$, and we have:

\begin{theorem}\label{theorem:pbt}
	Under \ref{eq:LGM}, and with \ref{eq:map} estimate  $\theta_{i|t}$, we have $\Sigma_{i|t}\preceq \Sigma_{i|s}$ ($\forall i\leq s \leq t$) and thus
	\begin{align*}
		\bbE \left[\| \theta_i - \theta_{i|t} \|_2^2 \right] \leq \bbE \left[\| \theta_i - \theta_{i|s} \|_2^2 \right].
	\end{align*}
\end{theorem}
\begin{proof}
	For any $k$, RTS begins with $\theta_{k|k}$ and computes the means $\theta_{k-1|k},\dots,\theta_{1|k}$ together their error covariances. For the proof, we need its update formula for the error covariances
	\begin{align}\label{eq:RTS-Sigma}
		\Sigma_{j-1|k} = \Sigma_{j-1|j-1} + L_{j-1} ( \Sigma_{j|k} -  \Sigma_{j|j-1}  ) L_{j-1}^\top, \tag{RTS-$\Sigma$}
	\end{align}
	which holds for all $2\leq j\leq k$ and where $L_{j-1}$ is known as the \textit{backward Kalman gain matrix} and is defined as $L_{j-1}:=\Sigma_{j-1|j-1} A_t^\top \Sigma_{j|j-1}^{-1}$. For a proof of \ref{eq:RTS-Sigma}, see, e.g., \cite[Section 8.2.3]{Murphy2023-book2}.
	
	
	Substitute $k=t$ into \ref{eq:RTS-Sigma}, and the equality suggests: If $ \Sigma_{j|t} \preceq \Sigma_{j|j-1}$ then $\Sigma_{j-1|t} \preceq \Sigma_{j-1|j-1}$. Combining this with the inequality $\Sigma_{j|j}\preceq \Sigma_{j|j-1}$  (as indicated by  \ref{eq:kf-correction}) gives:
	\begin{align*}
		\textnormal{If $ \Sigma_{j|t} \preceq \Sigma_{j|j}$, then $\Sigma_{j-1|t} \preceq \Sigma_{j-1|j-1}$.}
	\end{align*}
	Since  $\Sigma_{t|t} \preceq \Sigma_{t|t}$ and the above holds for $j=t,t-1,\dots,s$, we get $\Sigma_{s|t} \preceq \Sigma_{s|s}$.
	
	Substitute $k=t$ and $k=s$ into \ref{eq:RTS-Sigma}, take the difference of the two equalities, and observe that: 
	\begin{align*}
		\textnormal{If $\Sigma_{j|t} \preceq \Sigma_{j|s}$, then $\Sigma_{j-1|t} \preceq \Sigma_{j-1|s}$.}
	\end{align*}
	Since $\Sigma_{s|t} \preceq \Sigma_{s|s}$ and the above holds for $j=s,s-1,\dots,i$, we get $\Sigma_{i|t} \preceq \Sigma_{i|s}$ and finish. 
\end{proof}

While $\theta_{i|t}$ and $\theta_{i|s}$ are both optimal in the sense of \ref{eq:map}, $\theta_{i|t}$ has a \textit{smaller} error covariance, so we consider $\theta_{i|t}$ to be \textit{better} than $\theta_{i|s}$. As discussed, this is an indicator of positive backward transfer. 


\myparagraph{Summary} For multiple linear regression tasks with a linear task relationship (\ref{eq:LGM}), KF continually estimates the solution to the most recent task, and this estimate can furthermore be transferred backward to better solve previous tasks. While transferring between different tasks requires the task relationship to be known, it is possible to solve all tasks continually and estimate their relationship simultaneously (e.g., \cite{Titsias-ICLR2024}). It is furthermore possible to do so without inferring the task relationship (e.g., \cite{Vidal-Automatica2008}).


\section{Concluding Remarks}
Adaptive filtering and continual learning involve the art of balancing the past and present. This tutorial highlighted their shared principles and presented a family of interconnected methods (\cref{fig:method-connection}). While adaptive filtering in its classic sense is tailored for linear models, its ramifications in continual learning were evidenced by three ideas: linearization, layer-wise application, and combination with pre-trained models. Moreover, we consolidated the mathematical foundations of continual learning by revisiting adaptive filtering theory from a continual learning perspective. We believe a solid grasp of these ideas and mathematics is vital to pushing the theoretical and algorithmic boundaries of continual learning. 

Besides a few concrete open problems mentioned in the paper, we now comment on some general directions for future research. First, adaptive filtering has its roles in other signal processing topics (e.g., \textit{subspace tracking}, \textit{online dictionary learning}, \textit{kernel adaptive filtering}); how do they connect to continual learning (see, e.g., \cite{Balzano-Allerton2010,Liu-book2011,Ruvolo-ICML2013})? We have seen several algorithmic ideas for nonlinear models and deep networks; how about theorems (see, e.g., \cite{Peng-ICML2023,Min-CDC2022})? Finally, the tutorial here is concerned with the past and present, highlighting catastrophic forgetting; how about the future (see, e.g., \cite{De-CoLLA2023})? 


\begin{figure}
	\centering
	\begin{tikzpicture}[
		linear/.style={
			rectangle,
			draw=black,
			minimum width=1.5cm,
			minimum height=0.6cm,
			text centered,
		},
		nonlinear/.style={
			rectangle,
			draw=black,
			fill=gray!20,
			minimum width=1.5cm,
			minimum height=0.6cm,
			text centered,
			rounded corners
		},
		dnn/.style={
			rectangle,
			draw=black,
			minimum width=1.5cm,
			minimum height=0.6cm,
			text centered,
			rounded corners
		},
		arrow/.style={
			>=stealth
		}
		]
		
		\node[linear] (minnorm) {\ref{eq:min-norm}};
		\node[nonlinear, below=1cm of minnorm] (icl) {\ref{eq:ICL-AF}};

		\node[linear, left=1cm of icl] (apa) {\ref{eq:APA2}};
		\node[nonlinear, right=1cm of minnorm] (orfit) {\ref{eq:ORFit}};
		\node[nonlinear, below=1cm of orfit] (ogd) {\ref{eq:OGD-AF}};
		
		\node[linear, left=1cm of apa] (rls) {\ref{eq:RLS-AF}};
		
		\node[linear, above=0.4cm of apa] (lms) {\ref{eq:LMS}};
		
		\node[linear, above left= 0.5cm and -0.3cm of rls] (drls) {\ref{eq:D-RLS-AF}};

		\node[linear, left=1cm of rls] (kf) {\hyperref[eq:kf-correction]{\textcolor{cvprblue}{KF}}};
		\node[nonlinear, below=1cm of kf] (ekf) {EKF};
		\node[linear, left=1cm of ekf] (rts) {\hyperref[eq:RTS-Sigma]{\textcolor{cvprblue}{RTS}}};
		
		\node[nonlinear, below=1cm of rls] (lrls) {Linearized RLS};
		
		
		\node[dnn, below=1cm of apa] (ebp) {EBP};
		\node[dnn, below=1cm of icl] (owm) {OWM};
		\node[dnn, below=1cm of ogd] (gp) {\ref{eq:GP}};
		
		
		\draw[arrow] (minnorm) -- (-5.25,0) -- (rls);
		\draw[arrow] (drls) -- (rls);
		
		\draw[arrow] (minnorm) -- (apa);
		
		\draw[arrow] (minnorm) -- (orfit);
		\draw[arrow] (minnorm) -- (icl);
		\draw[arrow] (apa) -- (icl);
		\draw[arrow] (apa) -- (lms);
		\draw[arrow] (orfit) -- (icl);
		\draw[arrow] (orfit) -- (ogd);
		\draw[arrow] (rls) -- (apa);
		\draw[arrow] (rls) -- (kf);
		\draw[arrow] (rls) -- (lrls);
		\draw[arrow] (rls) -- (ebp);
		\draw[arrow] (rls) -- (owm);
		
		\draw[arrow] (rts) -- (kf);
		\draw[arrow] (ekf) -- (kf);
		\draw[arrow] (ekf) -- (lrls);
		
		\draw[arrow] (gp) -- (ogd);
		\draw[arrow] (gp) -- (owm);
		\draw[arrow] (ebp) -- (owm);
		
		\node[black,anchor=west] at (-13,0)  {\footnotesize{Methods originally proposed for:} };
		
		\draw[draw=black] (-12.5,-0.3) rectangle (-11.5,-0.7);
		\node[black,anchor=west] at (-11.5,-0.525)  {\footnotesize{linear (Gaussian) models} };
		
		\draw[fill=gray!20,  draw=black,rounded corners] (-12.5,-1) rectangle (-11.5,-1.4);
		\node[black,anchor=west] at (-11.5,-1.2)  {\footnotesize{nonlinear models} };
		
		\draw[draw=black,rounded corners] (-12.5,-1.7) rectangle (-11.5,-2.1);
		\node[black,anchor=west] at (-11.5,-1.925)  {\footnotesize{deep networks} };
	\end{tikzpicture}
	\caption{Summary of various methods and their relations  (connected by lines) described in the paper. \label{fig:method-connection} }
\end{figure}
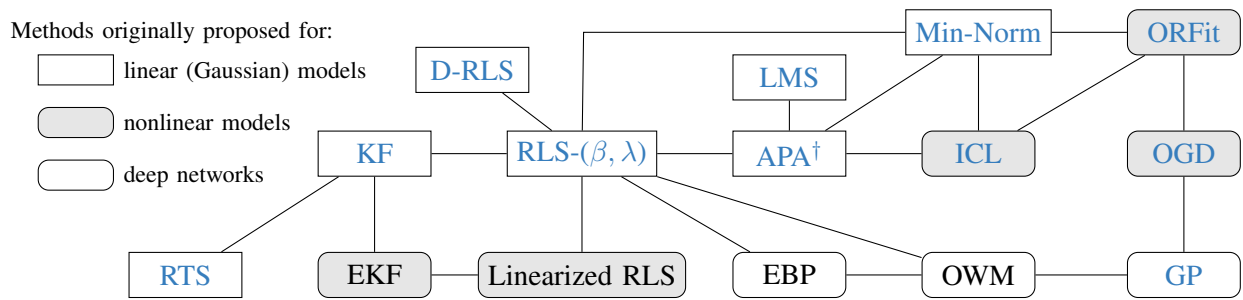




%
\myparagraph{Acknowledgment} This work is supported by the project ULEARN “Unsupervised Lifelong Learning”, funded by the Research Council of Norway (grant number 316080), and by NSF-Simons Research Collaborations on the Mathematical and Scientific Foundations of Deep Learning (NSF grant 2031985).

%






\bibliographystyle{IEEEtran}
\bibliography{Liangzu}
%
%
%
%
%
\section{Biographies}
\label{sec:bio}
{
\footnotesize \textbf{Liangzu Peng} is a PhD student at University of Pennsylvania. His advisor is Ren\'e Vidal. He received his M.S. degree from ShanghaiTech University and his B.S. degree from Zhejiang University. His current research focus is on continual learning.

\textbf{Ren\'e Vidal} (Fellow, IEEE) received his B.S. degree in electrical engineering from the Pontificia Universidad Cat\'olica de Chile in 1997, and his M.S. and Ph.D. degrees in electrical engineering and computer science from the University of California at Berkeley in 2000 and 2003, respectively. He is currently the Rachleff University Professor in the Departments of Electrical and Systems Engineering and Radiology at the University of Pennsylvania, where he also directs the Center for Innovation in Data Engineering and Science (IDEAS). He has coauthored the book Generalized Principal Component Analysis (Springer, 2016) and more than 300 articles in machine learning, computer vision, biomedical image analysis, signal processing, robotics and control. }

\end{document}